\documentclass[twoside]{article}

\usepackage[preprint]{aistats2027}

\usepackage{xr-hyper}
\usepackage{authblk}
\usepackage[utf8]{inputenc} 
\usepackage[T1]{fontenc}    

\usepackage{wrapfig,lipsum,booktabs}

\usepackage{dsfont}
\usepackage{paralist}
\usepackage{booktabs}       
\usepackage{nicefrac}       
\usepackage{multirow}
 \usepackage{xargs}
\usepackage{subcaption}
\usepackage{amsthm}
\usepackage{bm}
\usepackage{amssymb,mathrsfs}
\usepackage{amsfonts}
\usepackage{mathtools}
\usepackage{upgreek}
\usepackage{graphicx}
\usepackage{wrapfig}
\usepackage[dvipsnames]{xcolor}
\usepackage{soul}
\usepackage{pifont}
\usepackage{bbm}
\usepackage[colorlinks = true, citecolor = blue]{hyperref}
\usepackage{algpseudocode,algorithm,algorithmicx}
\usepackage{stmaryrd}
\usepackage{array}
\usepackage[inline]{enumitem}

\usepackage{tikz}
\usepackage{pgfplots}
\usepackage{aliascnt}
\usepackage{todonotes}
\usepackage{cleveref}
\usepackage{autonum}
\usepackage{accents}
\usepackage{etaremune}
\DeclareMathAlphabet{\mathpzc}{OT1}{pzc}{m}{it}

\definecolor{lightred}{rgb}{1, 0.8, 0.8}

\makeatletter
\newtheorem{theorem}{Theorem}
\crefname{theorem}{theorem}{Theorems}
\Crefname{Theorem}{Theorem}{Theorems}

\newtheorem*{lemma_nonumber*}{Lemma}

\newtheorem{lemma}{Lemma}

\newaliascnt{corollary}{theorem}

\aliascntresetthe{corollary}
\crefname{corollary}{corollary}{corollaries}
\Crefname{Corollary}{Corollary}{Corollaries}

\newaliascnt{proposition}{theorem}
\newtheorem{proposition}[proposition]{Proposition}
\aliascntresetthe{proposition}
\crefname{proposition}{proposition}{propositions}
\Crefname{Proposition}{Proposition}{Propositions}

\newaliascnt{definition}{theorem}

\aliascntresetthe{definition}
\crefname{definition}{definition}{definitions}
\Crefname{Definition}{Definition}{Definitions}

\newaliascnt{remark}{theorem}
\newtheorem{remark}[remark]{Remark}
\aliascntresetthe{remark}
\crefname{remark}{remark}{remarks}
\Crefname{Remark}{Remark}{Remarks}

\crefname{example}{example}{examples}
\Crefname{Example}{Example}{Examples}

\crefname{figure}{figure}{figures}
\Crefname{Figure}{Figure}{Figures}

\newtheorem{assumption}{\textbf{H}\hspace{-3pt}}
\crefformat{assumption}{{\textbf{H}}#2#1#3}

\crefformat{assumptionsup}{{\textbf{A}}#2#1#3}

\newcommand{\chname}[1]{\color{gray}\textbf{Solution}}

\newcommandx{\psr}[3][3=]{\left\langle#1,#2 \right\rangle_{#3}}
\newcommandx{\normr}[2][2=]{ \left\Vert#1 \right\Vert_{#2}}
\newcommandx{\psrLigne}[3][3=]{\langle#1,#2 \rangle_{#3}}
\newcommandx{\normrLigne}[2][2=]{ \Vert#1 \Vert_{#2}}

\newcommandx{\norm}[2][1=]{\ifthenelse{\equal{#1}{}}{\left\Vert #2 \right\Vert}{\left\Vert #2 \right\Vert^{#1}}}

\newcommand\probaMarkovTilde[2][2=]
{\ifthenelse{\equal{#2}{}}{{\widetilde{\mathbb{P}}_{#1}}}{\widetilde{\mathbb{P}}_{#1}\left[ #2\right]}}

\def\eqsp{\;}

\newcommand\sequence[3][2=,3=]
{\ifthenelse{\equal{#3}{}}{\ensuremath{\{ #1_{#2}\}}}{\ensuremath{\{ #1_{#2}, \eqsp #2 \in #3 \}}}}
\newcommand\sequenceD[3][2=,3=]
{\ifthenelse{\equal{#3}{}}{\ensuremath{\{ #1_{#2}\}}}{\ensuremath{( #1)_{ #2 \in #3} }}}

\newcommand\sequenceDouble[4][3=,4=]
{\ifthenelse{\equal{#3}{}}{\ensuremath{\{ (#1_{#3},#2_{#3}) \}}}{\ensuremath{\{  (#1_{#3},#2_{#3}), \eqsp #3 \in #4 \}}}}

\newcommand{\beq}{\begin{equation}}
\newcommand{\eeq}{\end{equation}}

\newcommandx{\gperthmc}[2][1=,2=]{\ifthenelse{\equal{#1}{}}{\Xi}{\ifthenelse{\equal{#2}{}}{\Xi_{h,#1}}{\Xi_{#2,#1}}}}
\newcommandx{\Phiverlet}[2][1=,2=]{\ifthenelse{\equal{#1}{}}{\Phi}{\Phi_{#1}^{\circ (#2)}}}
\newcommandx{\gpertub}[2][1=,2=]{\ifthenelse{\equal{#1}{}}{g}{g_{#1}^{#2}}}
\newcommandx{\Phiverletq}[2][1=,2=]{\ifthenelse{\equal{#1}{}}{\widetilde{\Phi}}{\widetilde{\Phi}_{#1}^{\circ (#2)}}}
\newcommandx{\Phiverletqi}[2][1=,2=]{\ifthenelse{\equal{#1}{}}{\bar{\Psi}}{\bar{\Psi}_{#1}^{(#2)}}}
\newcommandx{\Pkerhmc}[2][1=,2=]{\ifthenelse{\equal{#1}{}}{\mathrm{P}}{\mathrm{P}_{#1, #2}}}
\newcommandx{\tPkerhmc}[2][1=,2=]{\ifthenelse{\equal{#1}{}}{\tilde{\mathrm{P}}}{\tilde{\mathrm{P}}_{#1, #2}}}
\newcommandx{\PkerhmcD}[2][1=,2=]{\ifthenelse{\equal{#1}{}}{\mathrm{K}}{\mathrm{K}_{#1, #2}}}

\newcommand{\fracD}[2]{(#1/#2)}

\begin{document}

\twocolumn[

\aistatstitle{Extreme Binary Classification: Extreme Value Theory for Extreme Constraint on False Negative}

\aistatsauthor{ Samuel Gruffaz \And Muhammad Fawad \And Jaakko Nevalainen  }

\aistatsaddress{ Tampere university, Finland  } ]

\begin{abstract}
 While binary classification is one of the most extensively studied problems in machine learning,
  the regime in which the goal is to learn a classifier with an almost zero false negative rate remains largely unexplored.
   In this paper, we introduce the Extreme Binary Classification problem, where the objective is to learn a classifier whose false negative rate $\alpha$ is constrained by $\epsilon_{N_1}=o_{N_1\to\infty}(1/N_1)$, with $N_1$ denoting the number of positive examples in the training set.
 To address this problem, we propose a threshold adaptation method theoretically grounded in guarantees derived from Extreme Value Theory, together with a feature selection procedure based on a permutation test applied to sample maxima.
 Experimental results on four real-world datasets of varying sizes demonstrate that our approach compares favorably with state-of-the-art methods.
  In addition, we illustrate its interpretability through an application to a cancer screening dataset.

\end{abstract}





\section{Introduction}
Binary classification is one of the most fundamental and widely studied problems in machine learning \cite{kotsiantis2007supervised}.
 Its applications range from process automation and decision support to the analysis of the conditional distribution of a binary target variable $Y \in \{0,1\}$ given feature vectors $X \in \mathcal{X} \subset \mathbb{R}^d$.
   Before deploying a classifier in medical or industrial settings, it is crucial to assess and control the risks associated with misclassification \cite{tong2016survey,soize2017uncertainty,angelopoulos2024conformal,angelopoulos2023conformal}.

\begin{figure*}[t] 
    \centering 
    \includegraphics[width=0.32\textwidth]{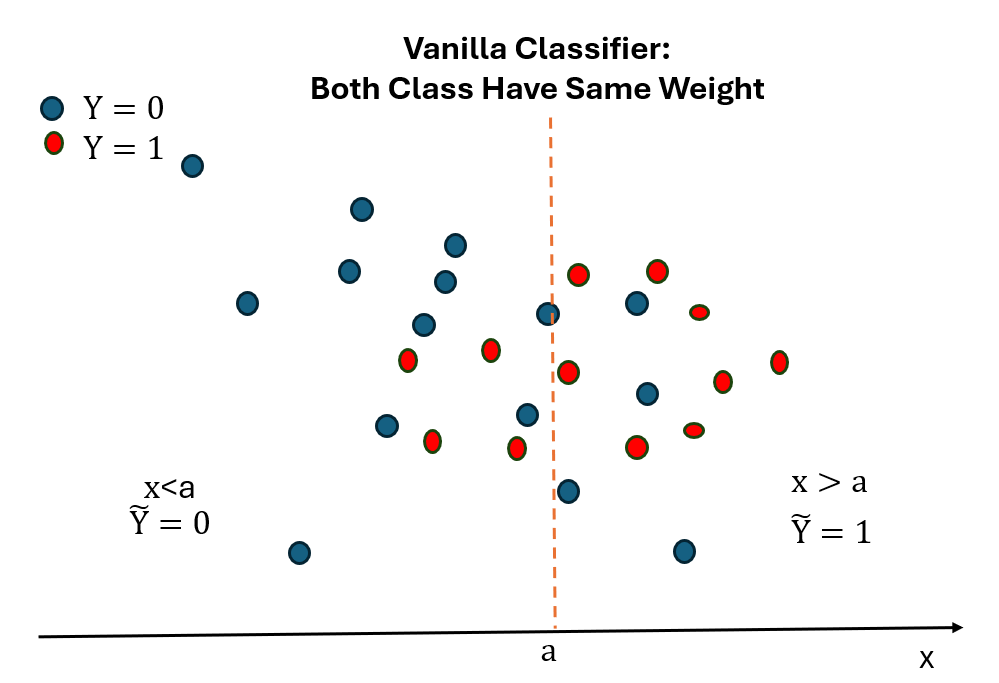}
    \hfill \includegraphics[width=0.32\textwidth]{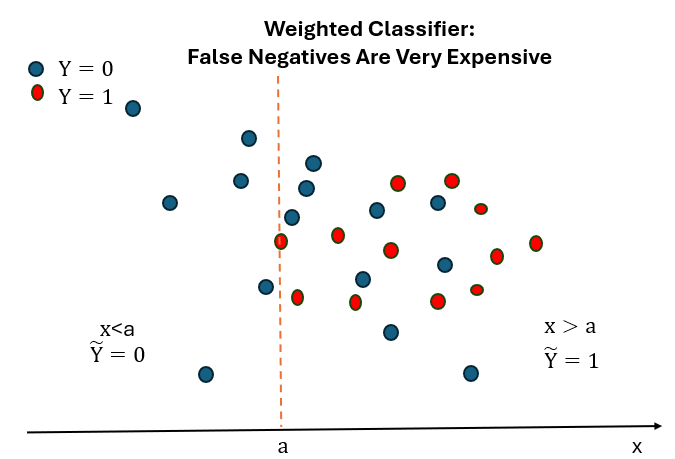} 
    \hfill \includegraphics[width=0.32\textwidth]{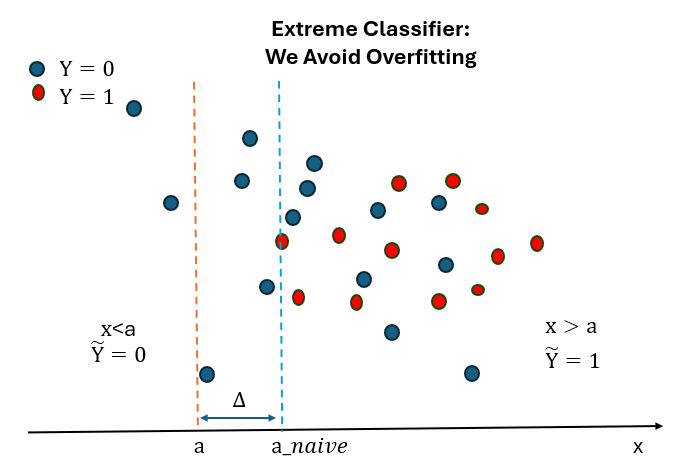} 
    \caption{ Illustration of the progression from a standard classifier (left) to a cost-sensitive classifier (middle) and finally to our extreme classifier (right). As the cost of false negatives increases, the decision threshold is shifted to increase sensitivity.
     Our extreme classifier further adjusts the threshold to provide strong control over false negatives while mitigating overfitting. }
      \label{fig:intro_banner} 
    \end{figure*}
In many real-world applications, binary classification is imbalanced in two distinct ways \cite{wang2021review}.
 First, one class $Y=1$ may be severely underrepresented in the available data. Second, the cost of misclassifying this class can be substantially higher than that of misclassifying the majority class.
  Such situations arise because the positive events of interest are often already rare due to existing preventive or monitoring procedures. Examples include equipment failures in industrial systems, cancer diagnoses in healthcare, or, in the most critical scenarios, patient mortality.
 In these applications, the false negative rate must be minimized despite having only a limited number $N_1$ of observations from the positive class $Y=1$.

Several approaches have been proposed to address this challenge, including data augmentation techniques \cite{wang2025comprehensive}, which oversample the minority class; cost-sensitive learning methods \cite{araf2024cost}, which reweight the training objective to emphasize costly errors; and threshold adaptation methods \cite{angelopoulos2024conformal,tong2018neyman},
 which estimate a decision threshold $\tau$ to control the false negative rate through a classifier of the form $ g(X)=\mathbbm{1}\!\left(f(X)>\tau\right)$. 
These methods are effective when the target false negative rate satisfies $\alpha \geq 1/N_1$, corresponding to a setting in which at least one false negative can be tolerated within the positive training samples.
 However, when the desired constraint becomes more stringent, namely $\alpha < 1/N_1$, existing approaches appear insufficient.
 Indeed, achieving such a level of reliability requires extending the decision boundary beyond the empirical support of the observed positive examples as depicted in \Cref{fig:intro_banner}.
The regime $\alpha < 1/N_1$ is particularly relevant in safety-critical applications, where even a single false negative may have severe consequences \cite{awan2019machine}. 
More broadly, it characterizes situations in which a positive prediction can be issued under extremely high confidence, which is also of independent interest in exploratory data analysis.
In this work, we refer to this setting as \textit{Extreme Binary Classification}.

 \paragraph{Contributions}
In this paper, we address this gap by leveraging Extreme Value Theory (EVT) \cite{de2006extreme} to satisfy the extreme false negative constraint $\alpha=o(1/N_1)$:
\begin{itemize}
    \item We introduce \textit{Extreme Binary Classification} and explain why EVT provides a natural framework for this problem (\Cref{sec:methodo}).
     Under mild assumptions on the tail behavior of the score $f(X)$, we show that the false negative rate can be controlled at the level $\alpha=o(1/N_1)$ with high probability using a tractable threshold estimator (\Cref{theorem:controlling_alpha}) provided that an exponential concentration inequality for the empirical mean excess applies (\Cref{prop:excess}).
    \item We develop a practical method inspired by \Cref{theorem:controlling_alpha} that can be applied either to a given score function $f$ or directly to the feature space $X$ in a decision-tree fashion, thereby improving interpretability. In particular, we introduce a feature-selection mechanism based on a statistical permutation test having a closed form (\Cref{perumation:maxima}) comparing the extrema of the empirical distributions of the positive and negative classes.
     \item We demonstrate the effectiveness of the proposed method on four real-world datasets, benchmarking it against state-of-the-art approaches \cite{angelopoulos2024conformal,tong2018neyman,araf2024cost,wang2025comprehensive,liu2009easyensemble}.
     We further illustrate its interpretability on a clinical cancer-screening dataset (\Cref{sec:experiments}), where false negatives may lead to missed diagnoses with potentially severe consequences. \end{itemize}

\section{Methodology}
\label{sec:methodo}
First, we formulate the Extreme Binary Classification problem (\Cref{subsec:EBC}). We then review the related literature and discuss the limitations of existing approaches in this setting (\Cref{subsec:literature_review}).
 Next, we revisit key results from Extreme Value Theory (\Cref{subsec:EVT}) and establish conditions under which the Extreme Binary Classification problem admits a solution (\Cref{theorem:controlling_alpha}, \Cref{prop:excess}).
 Finally, we introduce a classification algorithm and a feature selection strategy derived from our theoretical findings (\Cref{subsec:extreme_algo}).

 All proofs are provided in Appendix (\Cref{appendix:proofs}). For any $N \geq 1$, we let $[N]=\{1,\ldots,N\}$, and denote by $\mathbbm{1}$ the indicator function.

 \subsection{Extreme Binary Classification}
 \label{subsec:EBC}

Consider a binary classification dataset $ \mathcal{D}=(x_i,y_i)_{i\in[N]} \in (\mathbb{R}^d \times \{0,1\})^N, $
 where the observations are generated as i.i.d. samples from the random variables $(X,Y)$.
  Let $N_0$ and $N_1$ denote the numbers of samples satisfying $y_i=0$ and $y_i=1$, respectively.
   Our goal is to learn a classifier $ g:\mathbb{R}^d \mapsto \{0,1\}, $ with an almost zero false negative rate, ideally corresponding to at most one false negative in the worst-case scenario.
    Formally, we seek to enforce the constraint $ \mathbb{P}(g(X)\neq Y \mid Y=1)\leq\alpha \ll 1/N_1$.
     Under this regime, if the classifier is applied to $N_1$ new positive samples, the probability of observing at least one false negative is approximately $\alpha N_1 \ll 1$.
      The trivial classifier $g \equiv 1$ is not of interest, as our objective is also to minimize the number of false positives.
       Such stringent control of the false negative rate is highly desirable when the classifier is embedded within a decision-making system, where false negatives may lead to severe consequences.
        For example, in a clinical setting, a false negative prediction may prevent a patient from receiving appropriate treatment.
         Beyond safety-critical applications, this property is also valuable from an exploratory perspective. A classifier with an almost zero false negative rate identifies regions of the feature space that can be distinguished from the positive class with extremely high confidence, thereby revealing which characteristics reliably separate the distributions of $X \mid Y=0$ and $X \mid Y=1$.

 \paragraph{One dimensional case}
As a first step, consider the one-dimensional setting ($d=1$) and the classifier $ g_{a,b}(x)=\mathbbm{1}(x\in [a,b]), $ where $(a,b)\in \mathbb{R}^2 \cup \{-\infty,+\infty\}$ are two thresholds to be estimated.
 Although this family of decision functions is highly restrictive, it is also highly interpretable and can easily be integrated into human decision-making processes, such as clinical workflows.
  Moreover, when greater flexibility is required, the feature values can be defined as $x_i=f(z_i)$, where $f$ denotes the score produced by a pre-trained classifier.
   \vspace{0.2cm} A natural approach consists of solving the following empirical risk minimization problem:
    \begin{align} \label{eq:naive_classifier} 
        &a_{\text{naive}},b_{\text{naive}} \in \\
        &\arg\min_{a,b:\,\sum_{i:y_i=1}\mathbbm{1}\!\left(g_{a,b}(x_i)\neq 1\right)=0} \frac{1}{N}\sum_{i=1}^{N} \mathbbm{1}\!\left(g_{a,b}(x_i)\neq y_i\right).
     \end{align}
\begin{lemma}{\cite{shalev2014understanding}[Chapter 2, Ex 3]}
    \label{lemma:naive}
  $a_{\text{naive}}=\min \{ x_i: y_i=1\},b_{\text{naive}}=\max\{ x_i: y_i=1\} $ is a solution of \eqref{eq:naive_classifier}.
\end{lemma}

While straightforward, the estimator defined in \eqref{eq:naive_classifier} may be overly optimistic when the training set is small, as it ignores the finite-sample uncertainty associated with the empirical extrema, for instance if $X_i \sim \mathrm{Unif}([0,\theta])$, then
$\mathbb{E}\left(\max_{i=1,\ldots,n} X_i\right)\frac{n}{n+1}\theta$ \cite[Example 5.4.5]{casella2024statistical}.
 Consequently, the resulting classifier may underestimate the true false negative rate on unseen data. To account for this finite-sample effect, confidence margins $\Delta_a,\Delta_b>0$ should be introduced to ensure control of the false negative rate on the test set.
  Specifically, we consider the adjusted thresholds $ a=a_{\text{naive}}-\Delta_a, \qquad b=b_{\text{naive}}+\Delta_b $.
   To avoid the trivial solution $\Delta_a=\Delta_b=+\infty $, and thus $g_{a,b}=1$, we ask $0\leq \Delta_b<\max\{ x_i: y_i=0\}- \max\{ x_i: y_i=1\},\, 0\leq \Delta_a<\min\{ x_i: y_i=1\}- \min\{ x_i: y_i=0\} $ as a feasibility condition.
    In our suggested method, we will apply a permutation test on the extrema of the distributions $x_i: y_i=0,1$ to ensure that a solution $(\Delta_a,\Delta_b)$ exists.
  Using $\min_i X_i = -\max_i (-X_i)$, the estimation of $\Delta_a$ is equivalent to that of $\Delta_b$ after applying the transformation $x \mapsto -x$ to the sample $(x_i)$, so that we restrict our attention to $\Delta_b$ in the following.

   \subsection{Related works}
    \label{subsec:literature_review}
 In the following, we review existing solutions in the literature when $\alpha N_1=O(1)$ and we show that they fail in the regime $\alpha N_1=o(1) $.
\paragraph{Data augmentation and cost-sensitive learning. }
Data augmentation methods, popularized by SMOTE \cite{chawla2002smote}, and cost-sensitive learning approaches \cite{araf2024cost}, which increase the penalty associated with errors on the sensitive class, can be highly effective when the positive and negative classes are nearly separable.
In the more realistic setting where the class distributions overlap, these methods fail to achieve $\alpha N_1 = o(1)$.
 Because they operate within the convex hull of the empirical distribution \cite{chawla2002smote,araf2024cost}, they provide no mechanism for estimating the safety margins $\Delta_a,\Delta_b > 0$ required to control false negatives beyond the range of the observed positive samples.
  \paragraph{Neyman-Pearson classification. }
More rigorously, Neyman--Pearson classification addresses the following problem
\begin{equation} 
    \min_{g \in \mathcal{H} : \mathbb{P}(g(X)\neq Y \mid Y=1)\leq \alpha} \mathbb{P}(g(X)\neq Y \mid Y=0), 
\end{equation} 
where $\mathcal{H}$ denotes a class of decision functions and $\alpha>0$ is a prescribed upper bound on the false negative rate.
 When replacing the population quantity by its empirical counterpart, 
 a common strategy is to enforce a slightly more conservative constraint $\alpha'=\alpha-\epsilon/\sqrt{N_1},$ in order to account for finite-sample uncertainty \cite{rigollet2011neyman}.
  Here, $\epsilon(L,\delta)=O(1)$ is a constant depending on the regularity of the problem and the confidence level $1-\delta$, where $\delta>0$ denotes the probability of violating the false negative constraint.
  This approach implicitly requires that there exists $m>0$ s.t $\alpha\sqrt{N_1}\geq m$ so that the correction term remains smaller than the target error level.
   However, this assumption is violated in our setting, where $\alpha N_1=o(1)$.

  \paragraph{Threshold adaptation. }
  Most classifiers are expressed as $g(X)=\mathbbm{1}(\phi(X)\geq \eta)$ where $\phi$ is a \textit{scoring function} and $\eta>0$ is a threshold.
  Threshold adaptation methods suggest to control the false negative rate through a post-processing step rather than during training \cite{tong2018neyman,angelopoulos2024conformal,bates2021distribution}.
 Methods, such as the Neyman--Pearson umbrella algorithm \cite{tong2018neyman}, use a validation set containing positive samples to calibrate $\eta$, ensuring $\mathbb{P}(g(X)\neq Y\mid Y=1)\leq \alpha$ with confidence $1-\delta$ without any parametric assumption.
 However, in the extreme regime $\alpha=o(1/N_1)$, the resulting guarantees become vacuous since the probability $1-\delta\sim_{\alpha N_1\to 0} \alpha N_1$ necessarily tends to zero. In particular, the Neyman--Pearson umbrella algorithm  becomes equivalent to \eqref{eq:naive_classifier} in this regime as it relies on rank statistics, this is detailed in \Cref{discussion_NP}.

 This kind of post-processing step was generalized to other constraints by upper bounding the empirical risk, as suggested in conformal risk control \cite{angelopoulos2024conformal} or in risk-controlling prediction sets \cite{bates2021distribution}, but the same bottleneck regarding the lower bound on $\delta$ or $\alpha$ applies.




 \paragraph{Ensemble methods. } Ensemble methods \cite{liu2009easyensemble,vasconcelos2025mitigating} combine several of the aforementioned approaches and often aggregate predictions conservatively, for example by taking their maximum, in order to reduce the risk of false negatives. Although these methods tend to perform well in practice, they do not explicitly address the extreme regime considered in this work.

 The limitations of current methods appear to require $\alpha N_1\geq 1 $ to keep assumption general. If we aim to go beyond this limitation, we must make stronger assumptions.

 \paragraph{Parametric methods. }
Tong et al. who introduced the NP Umbrella algorithm attempted to remove the restrictions \cite{tong2020neyman} on $\alpha$ by assuming that the class-conditional distributions $X \mid Y=1$ and $X \mid Y=0$ are Gaussian and applying Linear Discriminant Analysis (LDA).
 However, this assumption is restrictive, and the covariance matrix may be poorly conditioned, leading to unstable estimates.
  More recently, \cite{pasche2026extreme} proposed using Extreme Value Theory (EVT) as a parametric assumption for constructing conformal prediction intervals in regression settings.
   However, they did not provide an end-to-end theoretical analysis of the resulting estimation procedure.
    In the following, we apply EVT to the extreme binary classification setting and, as far as possible, provide a complete theoretical analysis, from parameter estimation to the derivation of an upper bound on the false negative rate.

  \subsection{Background on Extreme Value Theory} 
  \label{subsec:EVT}
  In the remainder of this section, $(U_i)_{i\geq 0}$ denotes a sequence of i.i.d. continuous random variables with distribution $U \sim \mathbb{P}(\cdot \mid Y=1)$, and $U_{i,n}$ denotes the $i$th \textit{smallest} observation among the sample $U_1,\ldots,U_n$, for any $i\in[n]$ and $n\geq 1$, such that $b_{\text{naive}}(N_1)=U_{N_1,N_1}$.
  $U$ is used instead of $X$ to emphasize that we study the distribution of $X \mid Y=1$, rather than that of $X$.
   If we aim to construct a procedure that is more conservative than the naive solution of \Cref{lemma:naive}, we must make assumptions on the tail behaviour of the distribution of $U$ in order to obtain a meaningful estimate of the confidence width $\Delta_b$.
    In this setting, \textit{extreme value theory} (EVT) provides the natural framework \cite{de2006extreme,charras2013extreme,alves2014general}. 
     Ideally, we would like to upper bound the maximum of the distribution with high probability $1-\delta$, or at least an extreme quantile, so that for some sequence satisfying $\alpha_{N_1}=o(1/N_1)$, 
     \begin{align} 
        \label{eq:goal_estimate}
        &\mathbb{P}\!\left(\alpha_{N_1}^* \leq \alpha_{N_1}\right) \geq 1-\delta, \qquad \\
        &\alpha_{N_1}^* = \mathbb{P}\!\left( U_{N_1+1} > b_{\text{naive}}+\Delta_b(\delta) \mid (U_i)_{i\in[N_1]} \right). 
    \end{align}
 Here, $\alpha_{N_1}^*$ denotes the false negative rate of the classifier $g_{-\infty,b_{\text{naive}}+\Delta_b}$ learned from $(U_i)_{i\in[N_1]}$ and $U_{N_1+1}$ models a new observation lying outside the training set. 
 We note that EVT has recently been used to calibrate decision thresholds in medical testing \cite{pugh2024estimating}. Its application to classification can therefore be viewed as a natural generalization of these ideas.

   \paragraph{An asymptotic theorem}
   Extreme Value Theorem plays for empirical maxima the same role that the Central Limit Theorem (CLT) plays for empirical means: after an appropriate rescaling, both converge to universal limiting distributions that are largely independent of the underlying data-generating distribution.

   \begin{theorem}{\cite{charras2013extreme}[Theorem 2.3]}
    \label{theorem:EVT}
    If there exist two sequences $(a_n)_{n\geq 0},(b_n)_{n\geq 0}$ such that the rescaled maximum $M_n^{a,b}=(U_{n,n}-b_n)/a_n$ converge in distribution,
     there exist $(\xi,\mu,\sigma)\in \mathbb{R}^2\times \mathbb{R}_{>0}$ such that the CDF $H$ of the limiting distribution is:
    \begin{equation}
        H(x)=G(\xi,\mu,\sigma)(x)=\exp(-(1+\xi\frac{(x-\mu)}{\sigma})_+^{-\frac{1}{\xi}})
    \end{equation}
    using the limit $\lim_{\xi\to 0}(1+\xi x)_+^{-\frac{1}{\xi}}=\exp(-x)$ when $\xi = 0$. $G(\xi,\mu,\sigma)$ is the CDF of what is called a Generalized Extreme Value distribution (GEV). 
   \end{theorem}
   \begin{remark}
     $\xi$ is a parameter invariant under affine transformations of $M_n^{a,b}$, which makes it intrinsic to $\mathbb{P}(\cdot \mid Y=1)$. The right end-point is $M_F=\sup \{ x\in \mathbb{R}: G(\xi,\mu,\sigma)(x)<1\}=\mu+\sigma/|\xi|$ when $\xi<0$ and is infinite otherwise.
   \end{remark}
   If $\xi>0$ the tail is heavy such that $\Delta_b$ might need to be too high in practice to yield any true negatives, so we assume that $\xi\leq 0$ such that the maximum distribution has a light tail. These different behaviors are illustrated in \Cref{fig:EVT_illustration}.
   Contrary to the Central Limit Theorem, establishing convergence in distribution for $(M_n^{a,b})_{n\geq 0}$ is generally not straightforward. Nevertheless, such convergence holds for a wide range of commonly encountered distributions, including the normal, log-normal, and exponential distributions. In the remainder of this work, we therefore make the following mild assumption.
    \begin{assumption} \label{assumption:EVT}
         There exist two sequences $(a_n)$ and $(b_n)$ such that $(b_{\text{naive}}(N_1)-b_{N_1})/a_{N_1} $ converges in distribution, so that \Cref{theorem:EVT} applies with $\xi\leq 0$. 
        \end{assumption}

    \paragraph{Estimating the parameter}
   In order to establish \eqref{eq:goal_estimate} using an empirical estimator $\Delta_b$, we must estimate the parameters $\xi,\mu_{N_1},\sigma_{N_1} \in \mathbb{R}^2 \times \mathbb{R}_{>0}$ from the data so that $\mathbb{P}\!\left(b_{\text{naive}}(N_1)\leq x\right) \approx G(\xi,\mu_{N_1},\sigma_{N_1})(x) $. 
    These parameters cannot be reliably estimated from a single realization of the maximum $b_{\text{naive}}$. Instead, we must exploit the information contained in the $K$ largest observations, $ \left(U_{N_1-K+i,N_1}\right)_{i\in[K]}$,  in the hope of obtaining a consistent estimator. This idea is formalized by the following proposition, commonly known as the \emph{Peaks-Over-Threshold} (POT) principle.

    \begin{proposition}{\cite{kratz2019introduction}[Theorem 3]}
        \label{prop:POT}
    Assume \Cref{assumption:EVT}, then there exists $\xi\leq 0,\beta:\mathbb{R}\to \mathbb{R}_{>0}$ such that the tails are well approximated by a Generalized Pareto Distribution $G_{\xi,\beta(u)} $
    \begin{align}
        &\lim_{u\to M_F}\sup_{y\in [0,M_f-u]} |F_u(y)-G_{\xi,\beta(u)} (y)|=0,\\
        & F_u(y)=\mathbb{P}(U-u \leq y|U\geq u),\\
        & 1-G_{\xi,\beta(u)} (y)=\left(1+\xi \frac{y}{\beta(u)}\right)_+^{-\frac{1}{\xi}} \eqsp.
    \end{align}
    Moreover, the parameter $\xi$ has the same value as in the GEV given \Cref{theorem:EVT} and if $\xi<0$, the right end point of $G_{\xi,\beta}$ is $M_G=\beta/|\xi|$ such that $M_F\approx u+\beta /|\xi|$.
    \end{proposition}
    \Cref{prop:POT} is illustrated in \Cref{fig:POT_illustration}.
     Informally, the theorem states that, by using $u=U_{N_1-K,N_1}$ sufficiently large, the GEV shape parameter $\xi$ can be estimated by fitting a Generalized Pareto Distribution $G_{\hat{\xi},\hat{\beta}}$ to the samples $\left(U_{N_1-K+i,N_1}\right)_{i\in[K]}$.
     The choice of $K$ is not straightforward in practice, although \Cref{prop:POT} guarantees the existence of a suitable value.
      Furthermore, \textbf{when $\xi<0$}, the upper endpoint can be estimated by $\hat{M}_F=U_{N_1-K,N_1}+\hat{\beta}/|\hat{\xi}|$.
       In the small-sample setting, however, the estimation of $\xi$ may be unstable when $\xi<0$. 
        To address this issue, \cite{alves2014general} proposed the following endpoint estimator:
    \begin{align}
       & \hat{M}_F=U_{N_1,N_1}+\sum_{i=0}^{K-1} a_{i,K}(U_{N_1-K,N_1}-U_{N_1-K-i,N_1}), \\
        &\label{eq:end_point_estimate} \quad a_{i,K}=\log((K+i+1)/(K+i))/\log(2)
    \end{align}
    which is proven to converge to $M_F$ as long as $N_1\to \infty$, $K\to \infty$ and $K/N_1 \to 0$. As a rule of thumb, we verify the above condition by setting $K=\sqrt{N_1}$.
    \textbf{In the case of $\xi=0$}, $G_{0,\beta}=1-\exp(-\cdot/\beta)$ is the exponential of parameter $1/\beta$ and $\beta$ is estimated with the empirical mean excess, 
    \begin{equation}
        \label{eq:beta}
    \hat{\beta}=\frac{1}{K}\sum_{i=1}^K U_{N_1-K+i,N_1}-U_{N_1-K,N_1}\eqsp .
    \end{equation}

     which is directly used when computing the quantile of the distribution $ G_{0,\hat{\beta}}^{-1}(Q)=-\hat{\beta} \log(1-Q)$. This will be used in the following to estimate $\Delta_b$.


    \subsection{Extreme classification}
    \label{subsec:extreme_algo}

    We now derive a reasonable estimator $\Delta_b$ satisfying \eqref{eq:goal_estimate} under suitable assumptions.

     \begin{theorem}
        \label{theorem:controlling_alpha}
     Assume \Cref{assumption:EVT} and that the Generalized Pareto approximation is converging fast, i.e. there exists $\epsilon_0: \mathbb{R}_{>0}\mapsto \mathbb{R}$ s.t. $\lim_{u\to M_F} \epsilon_0(u)=0$ and
     \begin{equation}
        \label{eq:fast_converging}
     \sup_{y\in [0,M_F-y]} |F_u(y)-G_{\xi,\beta(u)} (y)| = \epsilon_0(u) \mathbb{P}(U\geq u),
     \end{equation} 
     moreover assume that we access an estimator $\Delta_b\left((U_{N_1-K+i,N_1})_{i\in[K]},\delta_{\text{est}}\right) $ such that for $K= \lfloor\sqrt{N_1}\rfloor+1,  Q_K=o_{K\to \infty}(1/K)$ and any $\xi\leq 0$, we have
     \begin{equation}
        \label{eq:estimator_cond}
           \mathbb{P}\left( (1-G_{\xi,\beta(U_{N_1-K,N_1})})(\Delta_b)\leq Q_K\right)\geq 1-\delta_{\text{est}}
        \end{equation}
        then, there exists $\epsilon_1(N_1)=o_{N_1\to \infty}(1/N_1)$, s.t for any $\delta_{\text{quant}}>0,\delta_{\text{est}}>0$, denoting by $\rho(\delta_{\text{quant}})=4\left(1+\sqrt{\log(2/\delta_{\text{quant}})/2}\right)^2$, 
     \begin{align}
        &\mathbb{P}\left(\mathbb{P}(U_{N_1+1}>b_{\text{naive}}+\Delta_b | (U_i)_{i\in [N_1]})\leq \epsilon_1(N_1)\rho(\delta_{\text{quant}}) \right) \\
        &\geq 1-\delta_{\text{quant}}-\delta_{\text{est}}
     \end{align}
     \end{theorem}
     Assumptions \Cref{assumption:EVT} and \eqref{eq:fast_converging} jointly imply that the tail of the positive-class distribution $\mathbb{P}(\cdot \mid Y=1)$ is well approximated by a GPD.
     The assumption \eqref{eq:fast_converging} is satisfied for GPD, or student distribution of parameter $\nu> 2$, but is not satisfied by certain distributions, such as the log-normal distribution, for which the convergence to the GPD limit is known to be slow.
      Consequently, the present analysis should be viewed as an optimistic scenario.
      \begin{remark}
        \label{remark:conservative_estimate}
      Note that the same result is obtained when using $U_{N_1-K,N_1}$ instead of $b_{\text{naive}}$. We nevertheless chose to retain $b_{\text{naive}}$ for consistency with our experiments, which favor conservative threshold estimates. 
      \end{remark}

    In the following proposition, we show how to construct an estimator $\Delta_b$ that verifies \eqref{eq:estimator_cond} under the assumption that only the \textit{mean} of the generalized pareto distribution is exponentially well approximated as $K,N_1\to\infty$.
    \begin{proposition}
        \label{prop:excess}
        Let $K\geq 3$. Assume that there exists $c>0$ such that the samples $D_i^{N_1}=U_{N_1-K+i,N_1}-U_{N_1-K,N_1} $ for any $i\in [K]$ verify the following concentration inequality for any $K\geq \sqrt{N_1}$, denoting by $\theta_k=c\sqrt{2\log(1/\delta_{\text{est}})/K}$,
        \begin{equation}
            \label{eq:concentration}
            \mathbb{P}\left( \mathbb{E}_{D\sim G_{\xi,\beta(U_{N_1-K,N_1})}}(D)-\sum_{i=1}^K \frac{D_i^{N_1}}{K}\geq \theta_k \right)\leq \delta_{\text{est}}
        \end{equation}
        then denoting by $\eta_K=-(1+\exp(-1))\log(1/(K\log(K))) $,
        \begin{equation}
            \label{eq:Delta_theo}
        \Delta_b= \left(\sum_{i=1}^K \frac{D_i^{N_1}}{K}+\theta_k\right)\eta_K
        \end{equation}
        verifies assumption \eqref{eq:estimator_cond}.
    \end{proposition}
    The interest of the previous proposition is to suggest an estimator independent of $\xi$ and which has a closed form.
    Such concentration inequality arises when $D_i^{N_1} $ is i.i.d given $U_{N_1-K,N_1}$ with sub-exponential tails and $|\mathbb{E}(D_i^{N_1})-\mathbb{E}_{D\sim G_{\xi,\beta(U_{N_1-K,N_1})}}(D)|\propto 1/\sqrt{K} $, which is very likely to be verified when $K$ and $N_1$ are large as rank statistics $\left(U_{N_1-K+i,N_1}\right)_{i\in[K]}$ are approximately independent and the mean of the GPD and $\mathbb{P}(U |U\geq U_{N_1-K,N_1}) $ are expected to match.
    The last insight remains quite challenging to prove rigorously, so we have preferred to use a concentration inequality that tackles the bottleneck of the estimation.
    In the idealized setting where $ D_i^{N_1} \overset{\mathrm{i.i.d.}}{\sim} \mathrm{Exp}(1/\beta)$,  
    the constant is given by $c=\sqrt{2}\,\beta$. In practice, however, the constant $c$ is generally unknown.
     In our numerical experiments, we use \Cref{prop:excess} as a rule of thumb and set $ \Delta_b = -\log(1-Q_b)\sum_{i=1}^{K} D_i^{N_1}/K=\hat{\beta}\log(1-Q_b), $ with $Q$ chosen according to the desired trade-off between the true negative rate and the false negative rate.


    We are now ready to state the practical algorithm.
    \paragraph{One dimensional Algorithm. }
    Since our entire reasoning becomes meaningless when $\max \{x_i:y_i=1\}\geq \max \{x_i:y_i=0\}$, we first apply a permutation test \cite{ernst2004permutation} using $T=\max\{x_i:y_i=0\}-\max\{x_i:y_i=1\}$ as test statistic to assess whether $T>0$ with $p<0.05$.
     The null distribution of $T$ is usually estimated by repeatedly permuting the class labels with Monte Carlo simulation, however we obtain a closed-form for the $p$-values making the test extremely fast $O(N)$ without approximation error, this is described in detail in \Cref{appendix:bootstrap}-(\Cref{perumation:maxima}). 
     If the maxima are not significantly different, we set $\Delta_b=\infty$.
      Otherwise, we apply the statistical test of \cite{alves2014general}[Theorem 6] to determine whether $H_0:\xi=0$ or $\bar{H}_0:\xi<0$ at the 99\% significance level.
       If $\xi=0$, we set 
       \begin{equation}
        \label{eq:practical_Delta_b}
       \Delta_b=-\hat{\beta}\log(1-Q_b)
       \end{equation}   using \eqref{eq:beta} and a confidence level $Q_b\in(0,1)$.
        Otherwise, we set \[ \Delta_b=\hat{M}_F-b_{\text{naive}} \] using \eqref{eq:end_point_estimate}.
        The only parameter that must be specified is the confidence level $Q_b\in(0,1)$.
         In practice, it can be selected by enforcing a minimum true negative rate $\mathrm{Card}\{x_i\geq b_{\text{naive}}+\Delta_b : y_i=0\}/(N-N_1)\geq \lambda$ on the training set,
          or by following the rule of thumb provided in \Cref{prop:excess}, namely, 
          \begin{equation} 
            Q_b\approx 1-\frac{1}{K\log(K)} \approx 1-\frac{1}{\sqrt{N_1}\log(\sqrt{N_1})},
         \end{equation}
          so that $N_1=100$ yields $Q_b\approx 0.95$ when aiming to achieve a false negative rate of order $o(1/N_1)$.
           The construction of $\Delta_a$ is entirely symmetric: it is obtained by applying the transformation $x\mapsto -x$ to the data and using a confidence level $Q_a\in(0,1)$. Each component of interest has been evaluated in an ablation study given in appendix \Cref{sec:ablation_study}.

           The previous algorithm doesn't strictly follow \eqref{eq:Delta_theo}, but yields similar results on synthetic experiments (\Cref{appendix:synthetic}) and verifies the false negative constraint $1/(N_1\log(N_1))$.
           We attribute this observation to the use of $b_{\text{naive}}$ rather than $U_{N_1-K,N_1}$ in the threshold estimation procedure (see \Cref{remark:conservative_estimate}).

          Since the assumptions of \Cref{theorem:controlling_alpha} are difficult, and in some cases impossible, to verify in practice, the theoretical results should be viewed primarily as motivation for the design of the algorithm rather than as prescriptions for every implementation detail.
           In particular, the requirement that the false negative rate scales as $o(1/N_1)$ is best interpreted as a guiding principle for selecting the classification procedure and its associated threshold, rather than as a precise quantitative target that must be achieved in practice.
           Choosing a target false negative rate of $o(1/N_1)$ encourages extrapolation beyond the training set through nonzero margins $\Delta_a$ and $\Delta_b$ as depicted in \Cref{fig:intro_banner}, whereas achieving a rate of only $O(1/N_1)$ may be possible without any extrapolation, i.e., by taking $\Delta_a = \Delta_b = 0$.

    \paragraph{Multidimensional Algorithm. }
    In the case of multidimensional features $X=(X^{i})_{i\in [d]} $, we repeat the previous process on each coordinate, this yields $d$ classifiers, $g_{a_1,b_1},\ldots,g_{a_d,b_d}$ and we define $ g\triangleq \prod_{i=1}^d g_{a_i,b_i}  $.
    Note that the permutation test (\Cref{perumation:maxima}) applied to each coordinate offers a \text{features-selection mechanism} revealing which features are interesting for a classification with an extreme constraint on false negative.
    The false negative rate is then at worst $\alpha d_{\text{select}}$ (\Cref{appendix:multidim}) if $d_{\text{select}}$ coordinates are effectively used. In our experiments, we chose not to correct for this effect to avoid being overly conservative when the coordinates are correlated.
    We called our method \textit{Extreme thresholding}.

    Using our method, we consider the following variants in the experiments, which differ according to the choice of input features $X$:
 \begin{itemize} 
    \item \textbf{(Log)} We fit a multivariate logistic regression on the training set and apply one-dimensional Extreme thresholding to the resulting logistic scores, i.e. $f(x)=\beta^\top x$ if $\beta\in \mathbb{R}^d$ is the logistic coefficient. 
    \item \textbf{(Mul)} We apply multidimensional Extreme thresholding directly to the feature coordinates.
     \item \textbf{(MulLog)} We apply multidimensional Extreme thresholding to the feature coordinates augmented with the logistic regression score.
     \end{itemize}
      For each method, we use $Q_b \in \{0.1, 0.9, 0.95\}$ when defining $\Delta_b$ in order to assess the impact of the confidence level. Since $25 \leq N_1 \leq 500$, depending on the dataset, $Q_b$ should theoretically lie in the interval $(0.875, 0.985)$ to ensure that the false negative rate scales as $o(1/N_1)$.
       Consequently, $Q_b = 0.90$ and $Q_b = 0.95$ are reasonable choices. In contrast, $Q_b = 0.10$ is included as a representative setting in which the false negative rate is not particularly stringent.

   \paragraph{Complexity} Thanks to our closed-form expression for the permutation test $p$-value (\Cref{appendix:bootstrap}) and the simplicity of the proposed procedure, the overall computational complexity is $O\bigl(d(N+N_1\log N_1)\bigr)$,  which is lower than the training complexity of a decision tree, whose best-case scaling is $O(dN\log N)$.
    The term $dN_1\log N_1$ arises from the sorting of the feature values in the positive class, which is required to compute the thresholds introduced in \Cref{prop:excess}. If a score function $f$ is learned as well on the training set to preprocess the features, the complexity bottleneck is there, which is verified in our experiments training a logistic regression score as detailed in \Cref{subsection:computation_time}.






\section{Experiments}
\label{sec:experiments}
In this section, we compare our method against state-of-the-art approaches on four real-world binary classification datasets.
 We evaluate performance in terms of the average number of false negatives ($\text{FN}$) and true negatives ($\text{TN}$) obtained on the test set from 5-fold cross-validation repeated with 10 different random seeds.
  We also report $ \text{R}(c)=\text{TN}/\max(c,\text{FN})$ for $c\in \{0.1,1\}$, which provides a measure of the gain obtained relative to the number of missed positive cases.  Note that $\text{TN}$ and $\text{FN}$ are averaged across the 5 folds to reduce the influence on $\text{R}(c=0.1)$ of an unusually low number of false negatives in any single fold.
  We use $\max(c,\text{FN})$ to avoid singularity and to limit the interest of reducing FN in the case of $c=1$.
   Finally, we illustrate the interpretability of our feature-selection mechanism on our in-house dataset, ScreenX arising from a cancer screening trial that remains unnamed for anonymity reasons, and compare it with the feature importance measures produced by Gradient Boosting and a shallow Decision Tree.
    \paragraph{Datasets} 
    We consider two large and highly imbalanced open-access datasets, Credit \cite{dal2017credit} ($N=284807$, $N_1=492$, $d=30$) and Corruption \cite{vasconcelos2021dataset} ($N=303036$, $N_1=428$, $d=14$), which are considered in the benchmark given in \cite{vasconcelos2025mitigating}, as well as two smaller datasets: ScreenCervical ($N=858$,$d=9$), and ScreenX ($N=117$, $d=12$) which are clinical datasets.
    
     For the Credit and Corruption datasets, the application context does not strongly penalize false negatives, and the classification task is primarily exploratory. In contrast, for ScreenCervical and ScreenX, false negatives may have severe clinical consequences.
      Consequently, the Extreme Binary Classification framework is particularly relevant in these settings.
       Additional details on the datasets are provided in \Cref{appendix:datasets}.
        In appendix, we tackle the Breast Cancer Wisconsin \cite{breastcancerwisconsin17} ($N=569$, $N_1=212$, $d=30$) and the Golub gene expression high dimensional dataset $(N=84,d>7000)$ \cite{golub1999molecular}, both are relatively well separated and are thus considered as limit cases. Note that on Golub we used a L1 regularization in Log and MulLog.

     \begin{figure*}[t] 
    \centering 
    \includegraphics[width=0.98\textwidth]{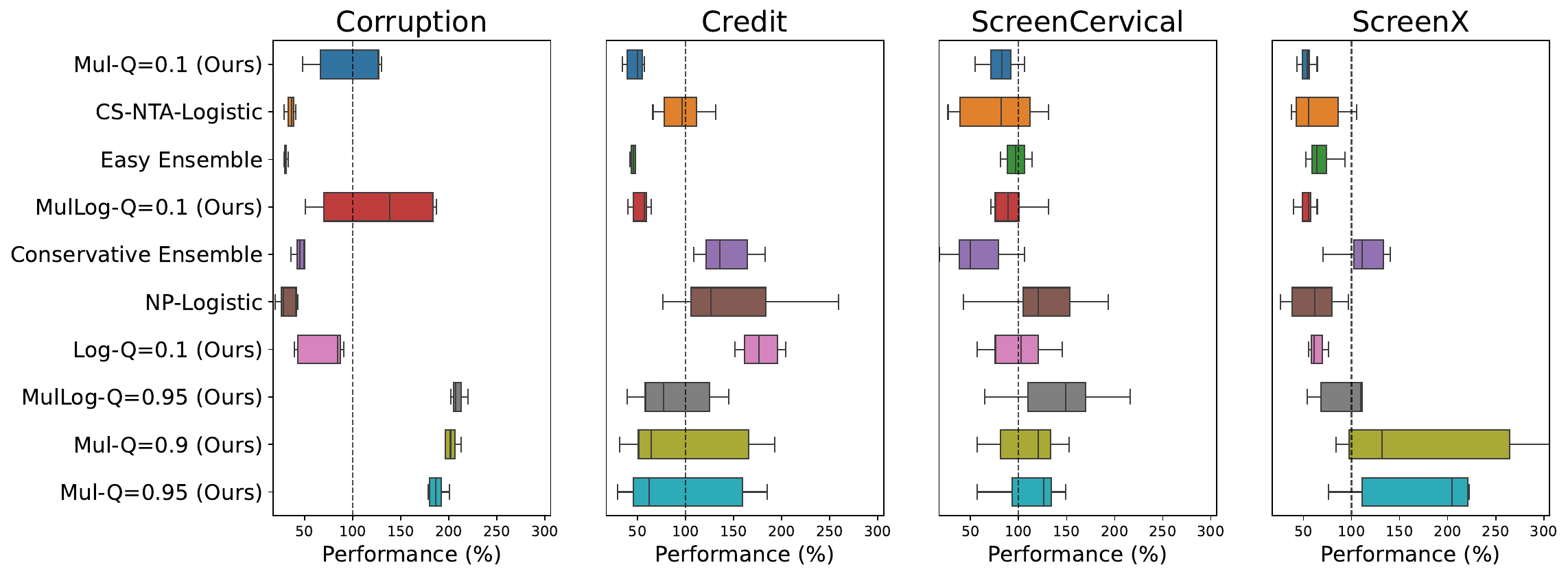}
    \caption{ Boxplot of $ \text{R}(c=0.1)$ normalized by the average performance per dataset for the Top 10 methods on Credit, Corruption, ScreenCervical and ScreenX in descending order, such that the best method on the benchmark considering the average rank is at the bottom. A performance of 200\% should be thus interpreted as performing two times better than the average. }
      \label{fig:bench_box_top10} 
    \end{figure*}


       \paragraph{Competitors} 
       We considered as broad a range of competitors as possible, prioritizing widely used methods with publicly available implementations and their recommended default hyperparameters, with one exception for SVC detailed in \Cref{sec:parameter}.
    
            \textbf{Data augmentation methods:} EasyEnsemble, Balanced Random Forest, and Balanced Bagging from the \texttt{imbalanced-learn} package \cite{lemaavztre2017imbalanced}.

             \textbf{Cost-sensitive (CS) methods:} Logistic Regression, XGBoost, and Support Vector Classifier (SVC) from \texttt{scikit-learn} and the XGBoost package \cite{xgboost}.
              For all cost-sensitive methods, we assign a weight of $w=1000$ to class $1$, as increasing the weight further did not affect the results.

                \textbf{Naive Threshold Adaptation (NTA):} We select the decision threshold that maximizes $\text{Ratio}$ on a validation set comprising 20\% of the training data.

                \textbf{Risk-Controlled and Neyman--Pearson methods:} Whenever applicable, we use Risk-Conformal threshold adaptation (RC) \cite{angelopoulos2024conformal} implemented through MAPIE \cite{taquet2022mapie} with probability $0.7$ to avoid excessive conservativeness while maintaining meaningful risk control, and the Neyman--Pearson umbrella algorithm (NP) \cite{tong2018neyman} implemented through \textit{nproc} \cite{nproc}, both applied to Logistic Regression and tree-based ensemble models, using a validation set comprising 20\% of the training data.
                 The target false-negative rate is chosen as the smallest value that yields at least one prediction different from class 1.

                  \textbf{Conservative ensemble:} An ensemble that returns the maximum prediction among a Decision Tree, a Logistic Regression model, and a Random Forest from \texttt{scikit-learn}, each calibrated using NTA. 

    \textbf{Isolation Forest:} The anomaly detection method proposed in \cite{liu2008isolation}, converted into a binary classifier through score thresholding.
     The threshold is set to the highest anomaly score observed among positive training examples.

       \paragraph{Results}
       Complete results are provided in \Cref{appendix:table} and computation time in \Cref{subsection:computation_time}. Runtime measurements were obtained on a CPU-only server with 4\,GB of RAM.
        For the sake of brevity, we report in \Cref{fig:bench_box_top10} only the top 10 methods induced by the average rank over all datasets considering \text{R}($c$), with $c=0.1$ since it better represents the goal of the paper.
        A similar figure is given in appendix for \text{R}($c$) with $c=1$, see \Cref{fig:bench_box_top10_1}, as a limit case of extreme classification where we tolerate one false negative on the test set on average.
                 Note that all of the top five methods achieve, on average, at most one false negative, with the exception of Conservative Ensemble on Credit, which yield more than one false negative on average.
                 On Breast Cancer (\Cref{table:BreastCancer}) and Golub (\Cref{table:Golub}), Log-$Q=0.1$ is among the Top 3, but nearly equivalent to all Top 10 methods, because the datasets are well separated, so there is no need for large confidence width $\Delta_a,\Delta_b$.

        \paragraph{Discussion.} First, the results show that \textit{Mul} (Ours) achieves the best performance on average on Corruption, Credit and ScreenX. 
        On Breast Cancer and Golub, the logistic regression score is particularly useful because the classes are linearly separated.
         Conversely, on the ScreenX or ScreenCervical dataset, the logistic score is not needed. 
         Overall, the results suggest using $Q \geq 0.9$ when the objective is to achieve a high ratio while maintaining an 
         almost negligible number of false negatives.
          In contrast, smaller values such as $Q=0.1$ may be preferable when false negatives are less critical.
         Indeed, among the top five methods for \text{R}($c$) with $c=1$ in \Cref{fig:bench_box_top10_1}, \textit{MulLog}-$Q=0.1$ ranks among the top three, alongside XGBoost and EasyEnsemble, with only marginal differences in performance.

          Second, in \Cref{fig:bench_box_top10}, our methods are particularly competitive on the ScreenX and Corruption datasets, and slightly underperforming on the Credit dataset which favors the logistic regression score.

          The Conservative Ensemble (ranked 6th) and EasyEnsemble (ranked 9th), identified as state-of-the-art approaches in \cite{vasconcelos2025mitigating}, achieve strong overall performance, particularly in terms of \text{R}($c$) with $c=1$.
           However, compared with \textit{MulLog}, these methods are less interpretable and considerably more computationally expensive. In contrast, CS-NTA-Logistic (ranked 8th) and NP-Logistic (ranked 5th) provide computationally efficient and interpretable alternatives.
           Nevertheless, a fundamental limitation of both methods is the lower bound imposed on $\alpha$, which restricts their applicability when optimizing \text{R}($c$) for small values of $c$, such as $c=0.1$.



        \paragraph{Interpretability of Feature Selection.}
We now investigate the impact of the different features on the predictions obtained from the Corruption dataset. Using 40 random train--test splits, with 10\% of the data reserved for testing, we compute and average the feature importances produced by Decision Trees, Random Forests, and XGBoost, the regression coefficients of a logistic regression model, and the frequency with which our permutation-based extrema test detects a significant difference between classes at the 5\% significance level. We additionally report the quantity $R(c=0.1)$ obtained by applying our classification algorithm independently to each one-dimensional feature.
 This metric provides a measure of predictive certainty for features that pass the permutation-based extrema test.
  For a fair comparison, all competing methods are trained in a cost-sensitive setting with a class weight of $w=5000$ assigned to false negatives.
  
  This procedure produces a table containing seven columns, one for each method, and fourteen rows, one for each feature, as reported in \Cref{tab:corruption_feature_importance} in the Appendix.
   We then compare the five highest-ranked features identified by each method in order to assess whether the different approaches agree on the most relevant predictors.
    The resulting rankings are presented in \Cref{table:top5_corruption}.
    To facilitate interpretation, as detailed in \Cref{appendix:feature_selection}, we group the features into seven categories: \emph{Salary}, \emph{Number of Positions}, \emph{Diversity of Positions}, \emph{Ownership Fraction}, \emph{Secondary Business Activities}, \emph{Duration of Ownership}, and \emph{Individual Restrictions}.

     \Cref{table:top5_corruption} shows that Random Forest, XGBoost, and Decision Tree models produce very similar rankings. All three methods identify salary-related features as the most important predictors, followed by features describing the number of positions held.
      Logistic regression yields a somewhat different ranking, assigning the highest importance to restriction-related variables, 
     followed by salary-related features.
      In contrast, our approach produces substantially different results.
      The permutation test systematically excludes salary-related features and instead ranks \emph{Diversity of Positions} and \emph{Number of Positions} as the most relevant categories.
       A more detailed analysis on \emph{Diversity of Positions} reveals that individuals who have occupied more than ten distinct positions were never observed in the corruption class. 
       One possible explanation is that highly mobile individuals spend less time in a single position and therefore have fewer opportunities to engage in corrupt activities.

        Other baselines suggest that lower salaries are associated with a reduced corruption risk.
         However, salary alone does not provide a sufficiently strong signal to yield highly certain predictions, which explains why it is not selected by our extrema-based procedure.
          Feature selection in the tails answers a fundamentally different question from standard feature-importance methods.
          
          Traditional approaches ask: \emph{Which features are most useful, on average, for predicting corruption risk?} In contrast, our permutation-based extrema test asks: \emph{Which features provide the highest degree of certainty for at least a subset of predictions?}

Finally, we verified that our method achieves a higher classification ratio on corruption than the competing methods while relying on only five feature-specific thresholds using only \emph{Number of Positions} and \emph{Diversity of Positions}.
 This makes our decision rule applicable by practitioners even when the machine-learning classifier itself is unavailable.
 Such interpretability is particularly appealing in educational and clinical settings.

\paragraph{Limitations.} Our approach has three main limitations: 1. When the original features do not provide a suitable representation for the classification task, a preprocessing step based on a scoring model becomes essential. In practice, even a simple logistic regression score can substantially improve performance. 
2. Since the threshold is defined through a maximum operator, the method is sensitive to label noise and therefore requires relatively clean data. We investigated replacing the maximum with a quantile-based threshold to improve robustness. However, this modification increased the number of false negatives in noise-free settings. We therefore retained the maximum-based formulation for both empirical and theoretical reasons. 
3. The structure of $g_{a,b}$ implicitly assumes that the feature distribution is unimodal. As a result, the method may struggle on highly multimodal distributions. In such cases, clustering the data before applying our approach may be necessary.
    However, this limitation was not particularly apparent in our experiments. 

\section{Conclusion}
 The proposed method proves effective at maximizing the ratio of true negatives to false negatives, a criterion that is particularly relevant in applications where false positives incur a relatively low cost compared with false negatives. 
In addition, our approach provides a simple and interpretable feature-selection mechanism, which enhances transparency and facilitates its adoption in clinical settings.
\section{AI Use Statement}
We used AI solely for code drafting, rephrasing, grammar correction, and generating the dataset descriptions in \Cref{appendix:datasets}. \Cref{appendix:datasets} was carefully reviewed and verified by the authors.


%






\bibliographystyle{plain}
\bibliography{ref}

\section*{Checklist}



\begin{enumerate}

  \item For all models and algorithms presented, check if you include:
  \begin{enumerate}
    \item A clear description of the mathematical setting, assumptions, algorithm, and/or model. Yes in \Cref{sec:methodo}.
    \item An analysis of the properties and complexity (time, space, sample size) of any algorithm. Yes at the end of \Cref{sec:methodo}.
    \item (Optional) Anonymized source code, with specification of all dependencies, including external libraries. Yes
  \end{enumerate}

  \item For any theoretical claim, check if you include:
  \begin{enumerate}
    \item Statements of the full set of assumptions of all theoretical results. Yes.
    \item Complete proofs of all theoretical results. Yes.
    \item Clear explanations of any assumptions. Yes.     
  \end{enumerate}

  \item For all figures and tables that present empirical results, check if you include:
  \begin{enumerate}
    \item The code, data, and instructions needed to reproduce the main experimental results (either in the supplemental material or as a URL). Yes, except that ScreenX dataset is not given for anonymity.
    \item All the training details (e.g., data splits, hyperparameters, how they were chosen). Yes in \Cref{sec:experiments} and \Cref{sec:parameter}.
    \item A clear definition of the specific measure or statistics and error bars (e.g., with respect to the random seed after running experiments multiple times). Yes.
    \item A description of the computing infrastructure used. (e.g., type of GPUs, internal cluster, or cloud provider). Yes.
  \end{enumerate}

  \item If you are using existing assets (e.g., code, data, models) or curating/releasing new assets, check if you include:
  \begin{enumerate}
    \item Citations of the creator If your work uses existing assets. Yes.
    \item The license information of the assets, if applicable. Yes.
    \item New assets either in the supplemental material or as a URL, if applicable. [Not Applicable]
    \item Information about consent from data providers/curators. [Not Applicable]
    \item Discussion of sensible content if applicable, e.g., personally identifiable information or offensive content. [Not Applicable]
  \end{enumerate}

  \item If you used crowdsourcing or conducted research with human subjects, check if you include:
  \begin{enumerate}
    \item The full text of instructions given to participants and screenshots. [Not Applicable]
    \item Descriptions of potential participant risks, with links to Institutional Review Board (IRB) approvals if applicable. [Not Applicable]
    \item The estimated hourly wage paid to participants and the total amount spent on participant compensation. [Not Applicable]
  \end{enumerate}

\end{enumerate}

\appendix 
\onecolumn





\section{Illustration EVT Theory}

\begin{figure*}[!h] 
    \centering 
    \includegraphics[width=0.9\textwidth]{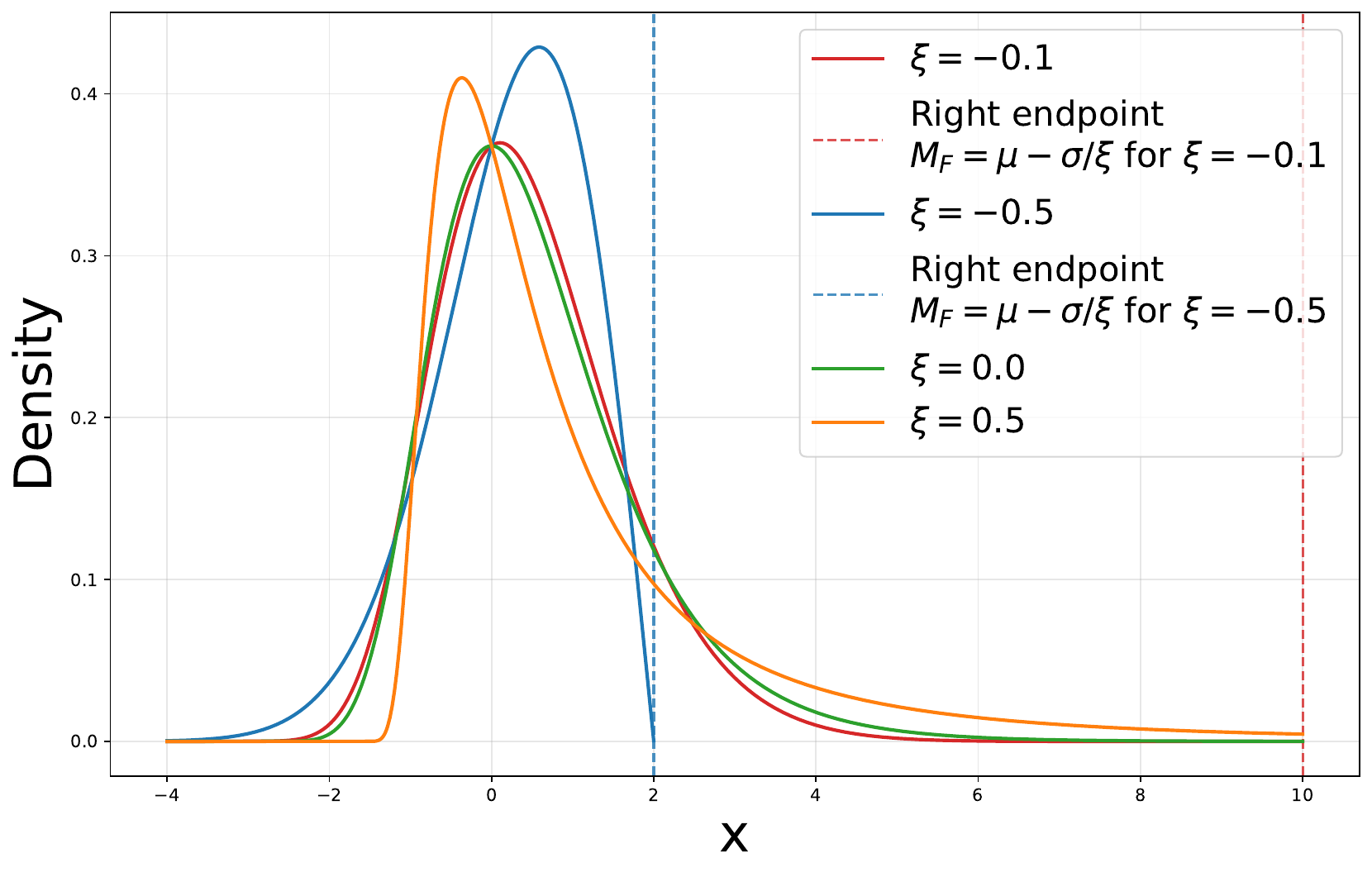}
    \caption{ Illustration of the Generalized Extreme Value (GEV) density for different values of the tail index $\xi$. }
      \label{fig:EVT_illustration} 
    \end{figure*}

\begin{figure*}[!h] 
    \centering 
    \includegraphics[width=0.9\textwidth]{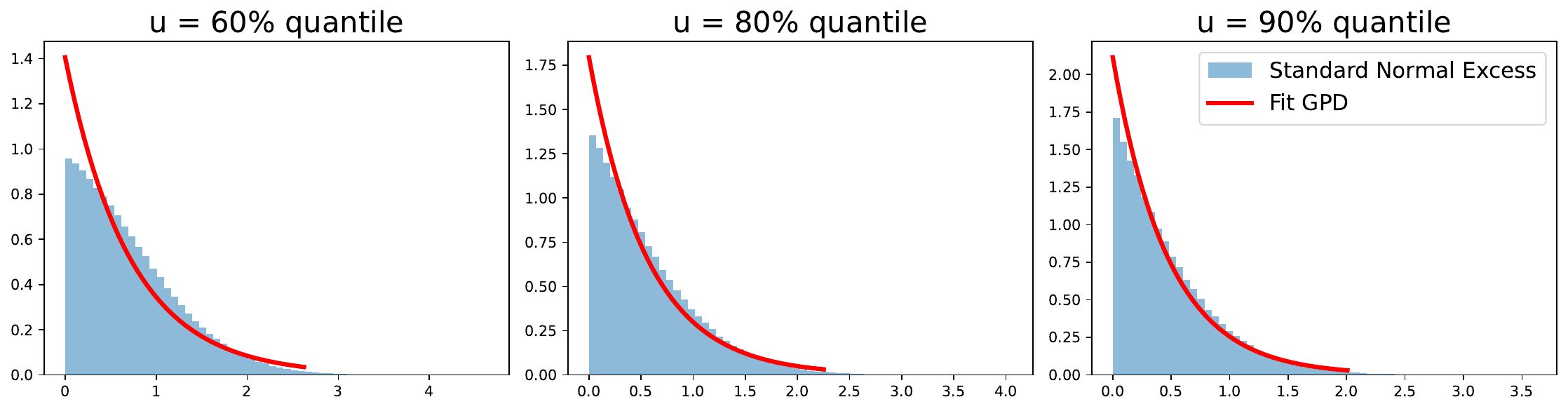}
    \caption{ Illustration of \Cref{prop:POT}. The histogram of the excesses $ E_k = N_k - Q_u \mid N_k > Q_u, \qquad N_k \sim \mathcal{N}(0,1), $ above the $u$-quantile $Q_u$ is well approximated by a GPD with $\xi=0$ when $u$ is sufficiently large. }
      \label{fig:POT_illustration} 
    \end{figure*}

\section{Experiments}



\subsection{Synthetic Experiments}
\label{appendix:synthetic}


In this section, we compare the performance of the theoretical estimator \eqref{eq:Delta_theo} for $\Delta_b$ with the practical estimator \eqref{eq:practical_Delta_b} used throughout our experiments.

 We consider a random variable $X_e$ following either an exponential distribution $\mathrm{Exp}(5)$ with scale parameter $5$,
  or a generalized Pareto distribution $G_{\xi,1}$ with tail index $\xi=-10^{-5}$, generated using \[ X_e=\frac{(1-U)^{-\xi}-1}{\xi}, \qquad U\sim\mathrm{Unif}([0,1]). \]
   Given a sample size $N_1$, we repeat the following procedure $10\,000$ times: 
   \begin{itemize} 
    \item Using $N_1$ i.i.d. 
    samples from $X_e$, we estimate $B(Q_b)=b_{\mathrm{naive}}+\Delta_b$ according to either \eqref{eq:Delta_theo} using $\delta_{est}=1/4, c=\sqrt{2}\beta$ or \eqref{eq:practical_Delta_b}
     for all \begin{equation}
         1-Q_b \in \left\{ \frac{1}{K\log(K)}, \frac{1}{(K-1)\log(K)}, \ldots, \frac{1}{\log(K)} \right\}, 
        \end{equation} where $K=\lfloor \sqrt{N_1}\rfloor+1$. 
        \item Then we compute using the exact survival function \[ \text{FN}=P(Q_b)=\mathbb{P}(X_e>B(Q_b)), \] 
        where $\text{FN}$ stands for false negative, to clarify the link with the manuscript.
    \end{itemize} 
    Finally, for each $N_1\in\{25,100,500\}$, we plot the estimated values of $P(Q_b)$ as a function of $Q_b$. We also report the variability across the $10\,000$ repetitions using shaded bands corresponding 
    to the IQC of $P(Q_b)$ over the $10\,000$ repetitions. 
    The results are shown in \Cref{fig:comparison_synthetic_expo,fig:comparison_synthetic_pareto}. The empirical densities of $\mathrm{Exp}(5)$ and $G_{\xi,1}$ are displayed in \Cref{fig:density_illu}.

     We also repeated the experiment with $X_e\sim\mathrm{Unif}([0,10])$ and obtained $FN=0$ in all repetitions, since $B(Q_b)>10$ for every estimated threshold.
      Indeed, for bounded distributions with tail index $\xi<0$ and $|\xi|<0.5$, the correction provided by \eqref{eq:practical_Delta_b} tends to be overly conservative. 
      For this reason, in our suggested algorithm, we instead rely on the endpoint estimator \eqref{eq:end_point_estimate} after applying the statistical test of \cite{alves2014general}[Theorem~6].
\begin{figure*}[t] 
\centering
\begin{subfigure}[t]{0.47\textwidth} 
\centering 
\includegraphics[width=\linewidth]{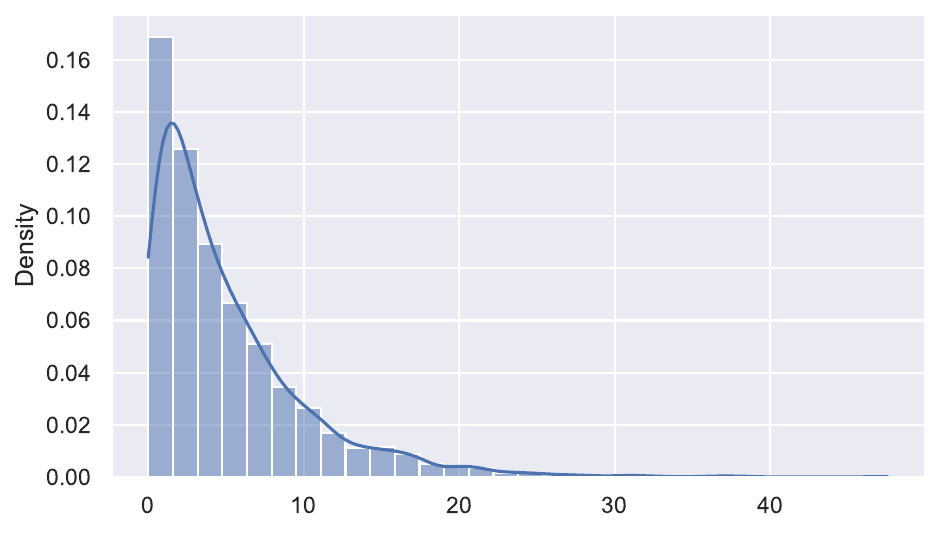}
\caption{$X_e\sim \text{Exp}(5)$} 
\label{fig:density_expo} 
\end{subfigure} 
\hfill 
\begin{subfigure}[t]{0.47\textwidth}
\centering
\includegraphics[width=\linewidth]{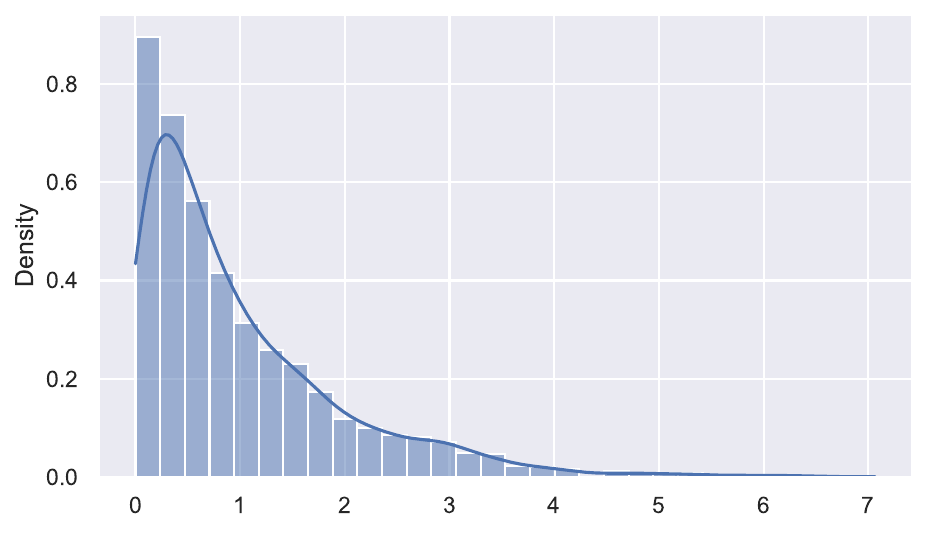} 
\caption{$X_e \sim G_{\xi,1}$ with $\xi=-10^{-5}$} 
\label{fig:density_gpd} 
\end{subfigure} 
\hfill 

\caption{ Illustration of the distribution used in the synthetic experiments using their histograms on 3000 samples. }
\label{fig:density_illu} 
\end{figure*}

\begin{figure*}[t] 
\centering
\begin{subfigure}[t]{0.32\textwidth} 
\centering 
\includegraphics[width=\linewidth]{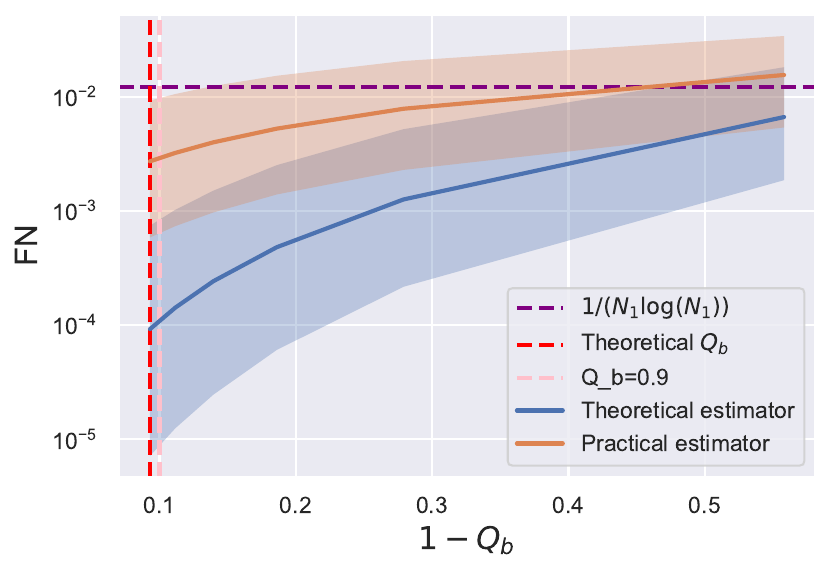}
\caption{$N_1=25$} 
\label{fig:exp_synthetic_expoN_25} 
\end{subfigure} 
\hfill 
\begin{subfigure}[t]{0.32\textwidth}
\centering
\includegraphics[width=\linewidth]{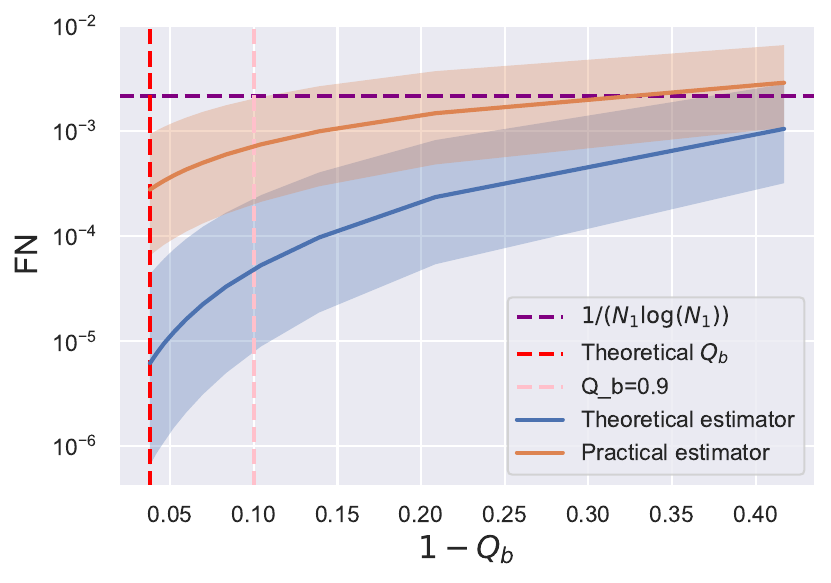} 
\caption{$N_1=100$} 
\label{fig:exp_synthetic_expoN_100} 
\end{subfigure} 
\hfill 
\begin{subfigure}[t]{0.32\textwidth}
\centering 
\includegraphics[width=\linewidth]{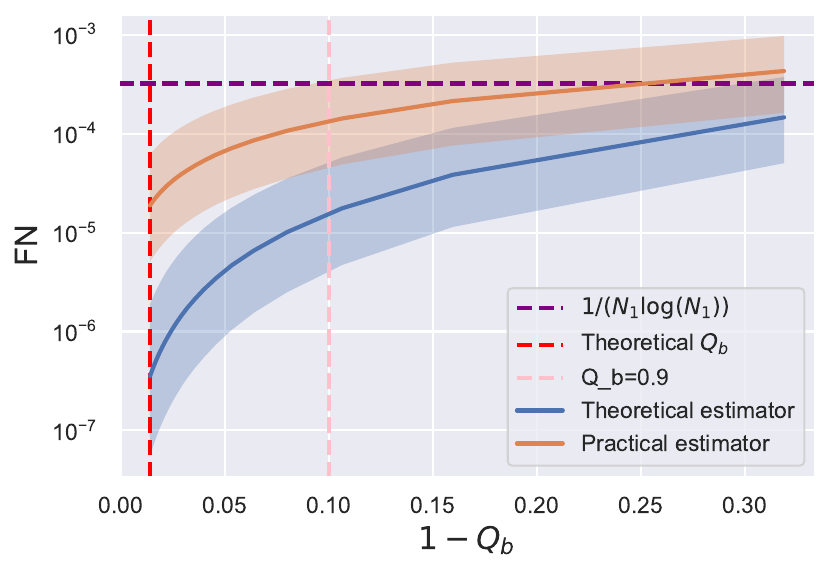}
\caption{$N_1=500$} 
\label{fig:exp_synthetic_expoN_500}
\end{subfigure} 
\caption{ Comparison of the false negative rate ($\text{FN}$) when $X_e \sim \text{Exp}(5)$. Notice that the practical estimator follows the same false negative constraint $1/(N_1\log(N_1)) $ as the theoretical one as long as $Q_b\geq 0.9$ for any $N_1\in \{25,100,500\}$.
We report the variability across $10\,000$ repetitions using shaded bands corresponding 
    to $\pm\sigma$, $\pm2\sigma$, and $\pm3\sigma$, where $\sigma$ denotes the standard deviation of $\text{FN}$ over the $10\,000$ repetitions.  }
\label{fig:comparison_synthetic_expo} 
\end{figure*}

\begin{figure*}[t] 
\centering
\begin{subfigure}[t]{0.32\textwidth} 
\centering 
\includegraphics[width=\linewidth]{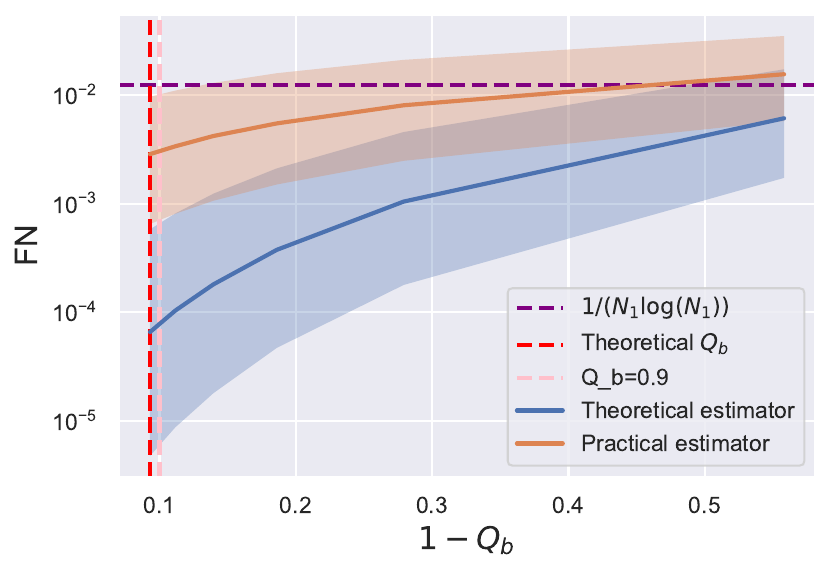}
\caption{$N_1=25$} 
\label{fig:exp_synthetic_paretoN_25} 
\end{subfigure} 
\hfill 
\begin{subfigure}[t]{0.32\textwidth}
\centering
\includegraphics[width=\linewidth]{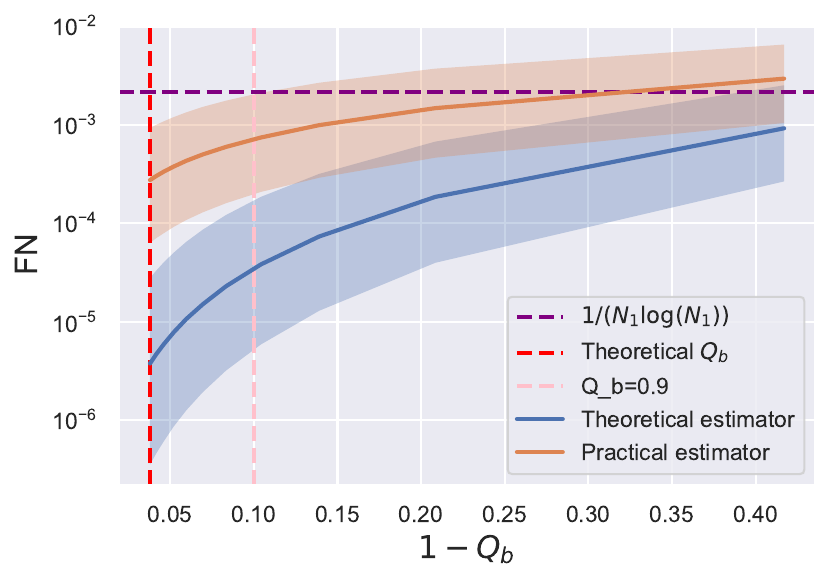} 
\caption{$N_1=100$} 
\label{fig:exp_synthetic_paretoN_100} 
\end{subfigure} 
\hfill 
\begin{subfigure}[t]{0.32\textwidth}
\centering 
\includegraphics[width=\linewidth]{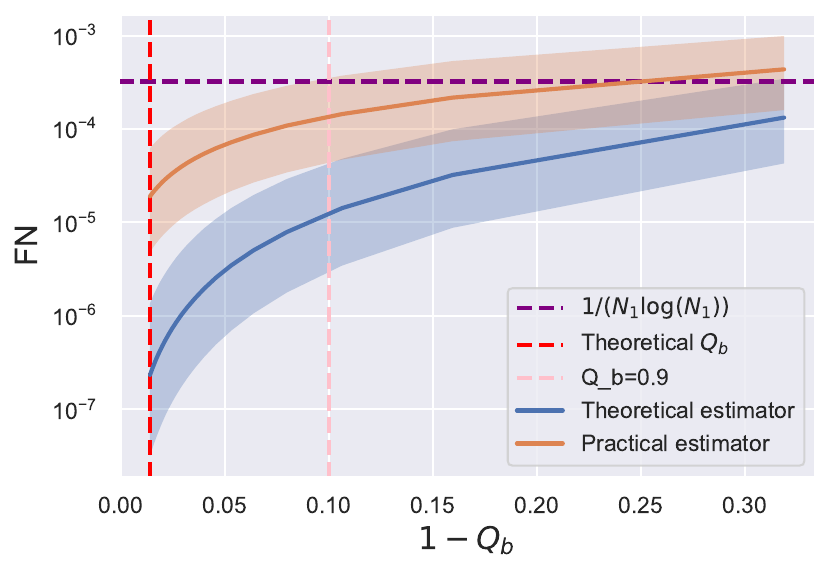}
\caption{$N_1=500$} 
\label{fig:exp_synthetic_paretoN_500}
\end{subfigure} 
\caption{ Comparison of the false negative rate ($\text{FN}$) when $X_e \sim G_{\xi,1}$ with $\xi=-10^{-5}$. Notice that the practical estimator follows the same false negative constraint $1/(N_1\log(N_1)) $ as the theoretical one as long as $Q_b\geq 0.9$ for any $N_1\in \{25,100,500\}$.
We report the variability across $10\,000$ repetitions using shaded bands corresponding 
    to $\pm\sigma$, $\pm2\sigma$, and $\pm3\sigma$, where $\sigma$ denotes the standard deviation of $\text{FN}$ over the $10\,000$ repetitions. }
\label{fig:comparison_synthetic_pareto} 
\end{figure*}

For both distributions the practical estimator follows the same false negative constraint $1/(N_1\log(N_1)) $ as the theoretical one as long as $Q_b\geq 0.9$ for any $N_1\in \{25,100,500\}$.

\newpage

\subsection{Datasets}
\label{appendix:datasets}

The \textbf{ScreenX} dataset considers a binary outcome $Y$ indicating the presence of aggressive cancer.
 The classifier's output is used to determine whether a biopsy should be performed based on preliminary diagnostic features $X$.
  The objective is to avoid unnecessary biopsies, as they constitute an invasive and burdensome procedure for patients,
   while maintaining a near-zero false negative rate.
    Indeed, failing to detect an aggressive cancer may delay treatment initiation and could ultimately lead to preventable mortality.

The \textbf{Corruption} dataset \cite{vasconcelos2021dataset} was constructed by integrating data from eight information systems of the Brazilian federal government and the Federal District.
 To ensure privacy protection, all personally identifiable information relating to civil servants and military personnel was removed, resulting in a fully anonymized dataset.
  The available features include information on political appointments, salaries, the number of positions held, and indicators of financial investments for each employee.
  The dataset is publicly available at \href{https://data.mendeley.com/ datasets/crpdknzswh/2}{this link}.

The \textbf{Credit} dataset \cite{dal2017credit}, available on Kaggle at \href{https://www.kaggle.com/datasets/mlg-ulb/creditcardfraud}{Kaggle}, corresponds to the "Credit Card Fraud Detection" task, whose objective is to identify fraudulent transactions from transaction-level features collected over a two-day period. For confidentiality reasons, all variables are numerical and have been transformed using PCA, except for the \texttt{Time} and \texttt{Amount} features, which are provided in their original form.

The \textbf{ScreenCervical} dataset \cite{cervical_cancer}, also called Cervical Cancer Risk Factors, was collected in the 'Hospital Universitario de Caracas' in Caracas, Venezuela.
 The dataset comprises demographic information, habits, and historic medical records of 858 patients. 
In the experiment, we consider only the following numeric features $X$ : Age, Number of sexual partners, First sexual intercourse,Num of pregnancies, Smokes (years), Smokes (packs/year),Hormonal Contraceptives (years), IUD (years), STDs (number). 
The target outcome $Y$ is whether the patient has a positive colposcopy (Hinselmann test) $Y=1$ or not $Y=0$.

The \textbf{Golub} gene expression dataset \cite{golub1999molecular} is a classic microarray dataset used to classify acute myeloid leukemia (AML) and acute lymphoblastic leukemia (ALL) from bone marrow and peripheral blood samples.

The \textbf{Breast Cancer Wisconsin} (Diagnostic) dataset \cite{breastcancerwisconsin17} is a widely used benchmark for binary classification in machine learning and statistics. 
 It consists of $N=569$ breast mass samples collected through fine needle aspiration (FNA), each labeled as either benign or malignant.
  The objective is to predict the diagnostic outcome from a set of quantitative descriptors extracted from digitized images of cell nuclei.
   The dataset contains 30 continuous features describing the morphology of cell nuclei.
    These features are derived from ten fundamental characteristics, namely radius, texture, perimeter, area, smoothness, compactness, concavity, concave points, symmetry, and fractal dimension.
     For each characteristic, the mean, standard error, and worst observed value are recorded, resulting in a total of 30 predictors. Among the 569 observations, 212 correspond to malignant tumors and 357 to benign tumors.
      Owing to its moderate size, low dimensionality, and strong predictive signal, the dataset has become a standard benchmark for evaluating classification algorithms. Many classical and modern methods achieve classification accuracies exceeding $95\%$ and area under the ROC curve (AUC) values above $0.98$.
\newpage
      \subsection{Details for competing methods}
\label{sec:parameter}

Training a standard SVC with an RBF kernel is computationally challenging on large datasets ($N\approx 300{,}000$). To obtain a scalable approximation, we first mapped the input features into a 200-dimensional randomized feature space using the Random Fourier Features implementation \texttt{RBFSampler(gamma="scale", n\_components=200, random\_state=seed)}.
 A linear SVC was then trained on these transformed features, yielding an efficient approximation of an RBF-kernel SVC while substantially reducing computational costs.

We set the maximum number of L-BFGS iterations for logistic regression to 1000 rather than the default value of 100 in order to ensure convergence to an optimal solution.

Note that, for the \textit{MulLog} and \textit{Log} methods, the logistic regression model is trained on the entire dataset.
 Similarly, the threshold on the logistic score coordinate is selected using the full dataset.
 Consequently, this threshold is expected to be slightly less conservative, as the logistic regression model is not estimated independently of the validation data.
  However, the test set, on which performance is ultimately evaluated, remains completely independent of this process, ensuring a fair comparison with the other methods.
   We chose not to introduce an additional data split in order to avoid further reducing the effective sample size.
   Moreover, standard logistic regression is known to be relatively robust to outliers. As a result, learning the decision threshold using EVT primarily exploits information contained in the tails of the distribution, which is largely complementary to the information driving the logistic regression model.
    In contrast, the \textit{Mul} method is not affected by this issue, as it does not involve learning a regression score.

We also investigated using the full dataset both for classifier training and for threshold estimation in the RC methods. However, this had little impact on the results. Therefore, for consistency with the existing literature, we retained the recommended procedure of separating the training and validation sets.

We were unable to perform a similar analysis for the NP methods due to limitations of the available API in \textit{nproc}, which does not provide a mechanism for separating the training and validation data during threshold estimation.
\newpage
\subsection{Ablation study}
\label{sec:ablation_study}
We conducted an ablation study to assess the contribution of each component of the proposed approach.
 Specifically, we compared naive thresholding with \texttt{Mul-Q=0.9 ()}, a version without feature selection or endpoint adaptation; \texttt{Mul-Q=0.9 (Selection)}, which incorporates only the feature selection mechanism; \texttt{Mul-Q=0.9 (END)}, which incorporates only endpoint adaptation; and \texttt{Mul-Q=0.9 (END+Selection)}, which combines both components.


\begin{table}
\caption{Relative improvement (\%) in $TN/\max(FN,0.1)$ with respect to the proposed method \texttt{Mul-Q=0.9 (END+Selection)}, which combines a feature selection mechanism (\emph{Selection}) with a tail-index statistical test to determine whether the endpoint estimator (\emph{END}) should be used.
 Although feature selection does not appear beneficial on the Corruption and ScreenCervical dataset, the results suggest that both mechanisms provide a clear advantage on average across datasets.
 This conclusion is even stronger if we consider $TN/\max(FN,1)$. }
\label{tab:improvement}
\begin{tabular}{lcccc}
\toprule
Method & Mul-Q=0.9 () & Mul-Q=0.9 (END) & Mul-Q=0.9 (Selection) & Naive \\
Dataset &  &  &  &  \\
\midrule
All & -31.9\% & -32.5\% & -12.6\% & -46.1\% \\
Cervical Cancer Screening & +63.7\% & +63.0\% & +0.0\% & -24.4\% \\
Corruption & +16.6\% & +32.9\% & -8.4\% & -38.7\% \\
Credit & -59.4\% & -69.3\% & -15.0\% & -50.3\% \\
ScreenX & -56.4\% & -56.5\% & +0.0\% & -71.9\% \\
\bottomrule
\end{tabular}
\end{table}

\newpage
\subsection{Complete results}
\label{appendix:table}

\begin{figure*}[t] 
    \centering 
    \includegraphics[width=0.98\textwidth]{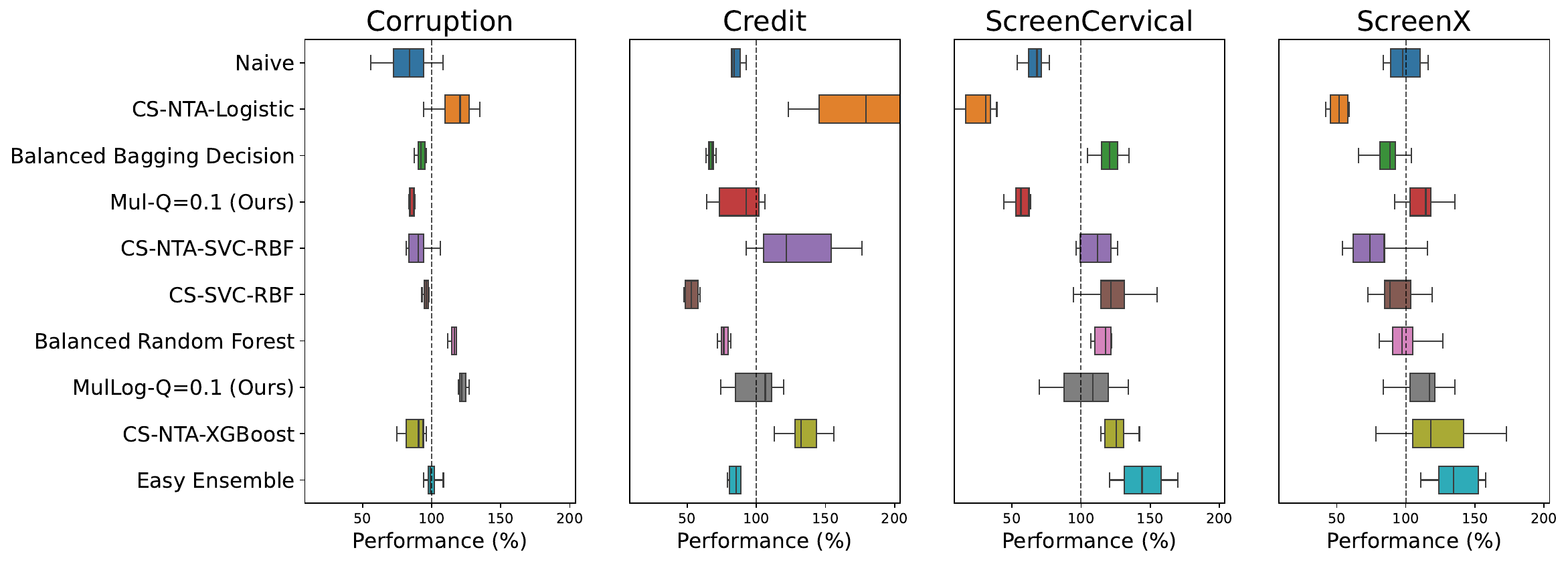}
    \caption{ Boxplot of $ \text{R}(c=1)$ normalized by the average performance per dataset for the Top 10 methods on Credit, Corruption, ScreenCervical and ScreenX in descending order, such that the best method on the benchmark considering the average rank is at the bottom. A performance of 200\% should be thus interpreted as performing two times better than the average. }
      \label{fig:bench_box_top10_1} 
    \end{figure*}


\begin{table}[!h]
    \label{table:ScreenX}
\caption{ScreenX benchmark results. We report the mean $\pm$ one standard deviation, computed over 10 repetitions with different random seeds, where each repetition corresponds to the mean performance across a 5-fold cross-validation. Method names in bold indicate that the number of false negatives was at most 1 across the 10 repetitions.
P(FN>1) records the frequency with which the average number of false negatives (FN) across the 5-fold cross-validation exceeds one over repeated runs.} \label{screenx_complete}
\begin{center}
\begin{tabular}{llllll}
\toprule
Dataset & \multicolumn{5}{c}{ScreenX} \\
 & FN & P(FN>1) & TN & TN/max(FN,0.1) & TN/max(FN,1) \\
\midrule
\textbf{Mul-Q=0.95 (Ours)} & 0.3 $\pm$ 0.2 & 0.0  & 7.4 $\pm$ 0.4 & 35.9 $\pm$ 21.3 & 7.4 $\pm$ 0.4 \\
\textbf{Mul-Q=0.9 (Ours)} & 0.4 $\pm$ 0.2 & 0.0  & 9.2 $\pm$ 0.5 & 33.2 $\pm$ 23.3 & 9.2 $\pm$ 0.5 \\
\textbf{Conservative Ensemble} & 0.2 $\pm$ 0.1 & 0.0  & 3.5 $\pm$ 0.6 & 20.9 $\pm$ 6.6 & 3.5 $\pm$ 0.6 \\
\textbf{MulLog-Q=0.95 (Ours)} & 0.5 $\pm$ 0.2 & 0.0  & 7.5 $\pm$ 0.4 & 19.4 $\pm$ 9.8 & 7.5 $\pm$ 0.4 \\
\textbf{CS-Logistic} & 0.3 $\pm$ 0.2 & 0.0  & 4.0 $\pm$ 0.4 & 14.3 $\pm$ 6.4 & 4.0 $\pm$ 0.4 \\
\textbf{NP-Random Forest} & 0.3 $\pm$ 0.2 & 0.0  & 3.0 $\pm$ 1.3 & 13.3 $\pm$ 8.9 & 3.0 $\pm$ 1.3 \\
Easy Ensemble & 1.1 $\pm$ 0.2 & 0.6  & 12.9 $\pm$ 0.6 & 11.7 $\pm$ 2.4 & 11.1 $\pm$ 1.4 \\
CS-NTA-XGBoost & 1.2 $\pm$ 0.4 & 0.7  & 12.3 $\pm$ 1.2 & 11.2 $\pm$ 5.1 & 9.9 $\pm$ 2.5 \\
\textbf{CS-NTA-Logistic} & 0.5 $\pm$ 0.2 & 0.0  & 4.5 $\pm$ 1.1 & 10.9 $\pm$ 4.5 & 4.5 $\pm$ 1.1 \\
\textbf{Log-Q=0.1 (Ours)} & 0.5 $\pm$ 0.2 & 0.0  & 4.6 $\pm$ 0.8 & 10.6 $\pm$ 1.8 & 4.6 $\pm$ 0.8 \\
NP-Logistic & 0.7 $\pm$ 0.4 & 0.2  & 6.2 $\pm$ 2.6 & 10.2 $\pm$ 4.3 & 5.8 $\pm$ 2.3 \\
MulLog-Q=0.1 (Ours) & 1.5 $\pm$ 0.2 & 1.0  & 13.4 $\pm$ 0.3 & 9.2 $\pm$ 1.3 & 9.2 $\pm$ 1.3 \\
Mul-Q=0.1 (Ours) & 1.5 $\pm$ 0.2 & 1.0  & 13.2 $\pm$ 0.2 & 9.2 $\pm$ 1.2 & 9.2 $\pm$ 1.2 \\
CS-XGBoost & 1.6 $\pm$ 0.2 & 1.0  & 14.4 $\pm$ 0.5 & 9.1 $\pm$ 1.3 & 9.1 $\pm$ 1.3 \\
Balanced Random Forest & 1.8 $\pm$ 0.3 & 1.0  & 14.7 $\pm$ 0.4 & 8.2 $\pm$ 1.4 & 8.2 $\pm$ 1.4 \\
Naive & 1.6 $\pm$ 0.2 & 1.0  & 12.6 $\pm$ 0.5 & 8.1 $\pm$ 1.0 & 8.1 $\pm$ 1.0 \\
CS-SVC-RBF & 1.6 $\pm$ 0.3 & 0.9  & 11.7 $\pm$ 0.7 & 7.8 $\pm$ 1.5 & 7.8 $\pm$ 1.5 \\
Balanced Bagging Decision & 2.2 $\pm$ 0.5 & 1.0  & 15.0 $\pm$ 0.7 & 7.0 $\pm$ 1.2 & 7.0 $\pm$ 1.2 \\
RC-Logistic & 2.1 $\pm$ 1.0 & 0.8  & 12.9 $\pm$ 3.6 & 6.9 $\pm$ 1.9 & 6.5 $\pm$ 1.7 \\
CS-NTA-SVC-RBF & 2.1 $\pm$ 0.8 & 0.9  & 12.1 $\pm$ 2.2 & 6.2 $\pm$ 1.6 & 6.2 $\pm$ 1.6 \\
RC-XGBoost & 2.6 $\pm$ 0.4 & 1.0  & 15.3 $\pm$ 1.2 & 5.9 $\pm$ 0.7 & 5.9 $\pm$ 0.7 \\
Isolation Forest & 2.4 $\pm$ 0.6 & 0.9  & 12.4 $\pm$ 1.8 & 5.4 $\pm$ 1.3 & 5.4 $\pm$ 1.3 \\
\textbf{Log-Q=0.95 (Ours)} & 0.3 $\pm$ 0.1 & 0.0  & 0.8 $\pm$ 0.7 & 2.8 $\pm$ 1.3 & 0.8 $\pm$ 0.7 \\
\bottomrule
\end{tabular}
\end{center}
\end{table}
\newpage
\begin{table}[!h]
    \label{table:Corruption}
\caption{Corruption benchmark results. We report the mean $\pm$ one standard deviation, computed over 10 repetitions with different random seeds, where each repetition corresponds to the mean performance across a 5-fold cross-validation. Method names in bold indicate that the number of false negatives was at most 1 across the 10 repetitions.
P(FN>1) records the frequency with which the average number of false negatives (FN) across the 5-fold cross-validation exceeds one over repeated runs.} \label{corruption_complete}
\begin{center}
\begin{tabular}{llllll}
\toprule
Dataset & \multicolumn{5}{c}{Corruption} \\
 & FN & P(FN>1) & TN & TN/max(FN,0.1) & TN/max(FN,1) \\
\midrule
\textbf{MulLog-Q=0.95 (Ours)} & 0.1 $\pm$ 0.2 & 0.0  & 956.9 $\pm$ 64.0 & 8633.6 $\pm$ 2391.6 & 956.9 $\pm$ 64.0 \\
\textbf{Mul-Q=0.95 (Ours)} & 0.1 $\pm$ 0.2 & 0.0  & 859.7 $\pm$ 67.3 & 7741.0 $\pm$ 2143.0 & 859.7 $\pm$ 67.3 \\
\textbf{Mul-Q=0.9 (Ours)} & 0.1 $\pm$ 0.2 & 0.0  & 940.4 $\pm$ 70.1 & 7724.2 $\pm$ 2932.8 & 940.4 $\pm$ 70.1 \\
\textbf{MulLog-Q=0.1 (Ours)} & 0.4 $\pm$ 0.2 & 0.0  & 1668.3 $\pm$ 59.2 & 5724.1 $\pm$ 2700.7 & 1668.3 $\pm$ 59.2 \\
\textbf{Mul-Q=0.1 (Ours)} & 0.3 $\pm$ 0.1 & 0.0  & 1163.9 $\pm$ 46.1 & 4512.2 $\pm$ 1548.1 & 1163.9 $\pm$ 46.1 \\
\textbf{Log-Q=0.1 (Ours)} & 0.3 $\pm$ 0.1 & 0.0  & 770.7 $\pm$ 40.3 & 3096.1 $\pm$ 1042.4 & 770.7 $\pm$ 40.3 \\
Conservative Ensemble & 2.5 $\pm$ 1.0 & 0.9  & 5343.1 $\pm$ 1439.6 & 2324.5 $\pm$ 837.1 & 2324.5 $\pm$ 837.1 \\
\textbf{NP-Logistic} & 0.5 $\pm$ 0.3 & 0.0  & 782.7 $\pm$ 227.8 & 1918.2 $\pm$ 1446.4 & 782.7 $\pm$ 227.8 \\
CS-NTA-Logistic & 6.0 $\pm$ 2.6 & 1.0  & 9374.2 $\pm$ 2782.3 & 1673.9 $\pm$ 376.2 & 1673.9 $\pm$ 376.2 \\
Balanced Random Forest & 26.3 $\pm$ 0.9 & 1.0  & 41566.7 $\pm$ 237.9 & 1580.9 $\pm$ 53.3 & 1580.9 $\pm$ 53.3 \\
\textbf{Log-Q=0.95 (Ours)} & 0.0 $\pm$ 0.0 & 0.0  & 149.8 $\pm$ 12.5 & 1498.0 $\pm$ 125.2 & 149.8 $\pm$ 12.5 \\
Easy Ensemble & 31.0 $\pm$ 1.5 & 1.0  & 41765.2 $\pm$ 202.7 & 1351.6 $\pm$ 60.9 & 1351.6 $\pm$ 60.9 \\
CS-SVC-RBF & 31.8 $\pm$ 0.8 & 1.0  & 41469.2 $\pm$ 465.7 & 1304.8 $\pm$ 36.9 & 1304.8 $\pm$ 36.9 \\
CS-Logistic & 23.2 $\pm$ 0.6 & 1.0  & 29918.4 $\pm$ 349.8 & 1291.4 $\pm$ 32.7 & 1291.4 $\pm$ 32.7 \\
Balanced Bagging Decision & 40.6 $\pm$ 1.3 & 1.0  & 50711.5 $\pm$ 176.0 & 1249.6 $\pm$ 40.9 & 1249.6 $\pm$ 40.9 \\
CS-NTA-SVC-RBF & 17.2 $\pm$ 6.0 & 1.0  & 21190.4 $\pm$ 7784.2 & 1227.5 $\pm$ 107.8 & 1227.5 $\pm$ 107.8 \\
Naive & 1.6 $\pm$ 0.5 & 0.9  & 1743.6 $\pm$ 152.8 & 1220.9 $\pm$ 472.7 & 1124.4 $\pm$ 247.0 \\
RC-XGBoost & 28.9 $\pm$ 4.7 & 1.0  & 34627.0 $\pm$ 4942.9 & 1202.3 $\pm$ 80.0 & 1202.3 $\pm$ 80.0 \\
CS-NTA-XGBoost & 24.1 $\pm$ 5.6 & 1.0  & 28376.7 $\pm$ 5600.0 & 1191.2 $\pm$ 103.5 & 1191.2 $\pm$ 103.5 \\
CS-XGBoost & 57.9 $\pm$ 1.2 & 1.0  & 58134.2 $\pm$ 60.8 & 1005.1 $\pm$ 20.8 & 1005.1 $\pm$ 20.8 \\
NP-Random Forest & 69.5 $\pm$ 1.4 & 1.0  & 60487.0 $\pm$ 9.9 & 870.1 $\pm$ 17.9 & 870.1 $\pm$ 17.9 \\
RC-Logistic & 84.0 $\pm$ 0.1 & 1.0  & 60478.0 $\pm$ 10.5 & 719.6 $\pm$ 1.0 & 719.6 $\pm$ 1.0 \\
Isolation Forest & 15.1 $\pm$ 13.0 & 1.0  & 10067.1 $\pm$ 9197.1 & 614.3 $\pm$ 78.4 & 614.3 $\pm$ 78.4 \\
\bottomrule
\end{tabular}

\end{center}
\end{table}
\newpage
\begin{table}[!h]
    \label{table:Credit}
\caption{Credit benchmark results. We report the mean $\pm$ one standard deviation, computed over 10 repetitions with different random seeds, where each repetition corresponds to the mean performance across a 5-fold cross-validation.
Method names in bold indicate that the number of false negatives was at most 1 across the 10 repetitions. 
P(FN>1) records the frequency with which the average number of false negatives (FN) across the 5-fold cross-validation exceeds one over repeated runs.} \label{credit_complete}
\begin{center}
\begin{tabular}{llllll}
\toprule
Dataset & \multicolumn{5}{c}{Credit} \\
 & FN & P(FN>1) & TN & TN/max(FN,0.1) & TN/max(FN,1) \\
\midrule
\textbf{Log-Q=0.1 (Ours)} & 0.2 $\pm$ 0.1 & 0.0  & 4919.0 $\pm$ 557.0 & 25358.5 $\pm$ 7005.5 & 4919.0 $\pm$ 557.0 \\
NP-Logistic & 0.7 $\pm$ 0.4 & 0.1  & 12196.0 $\pm$ 2638.3 & 22249.4 $\pm$ 11459.9 & 11550.8 $\pm$ 1985.2 \\
\textbf{Log-Q=0.95 (Ours)} & 0.2 $\pm$ 0.1 & 0.0  & 4037.6 $\pm$ 580.3 & 20752.3 $\pm$ 5598.3 & 4037.6 $\pm$ 580.3 \\
Conservative Ensemble & 1.7 $\pm$ 0.4 & 1.0  & 33094.6 $\pm$ 3202.6 & 20076.2 $\pm$ 3801.2 & 20076.2 $\pm$ 3801.2 \\
Mul-Q=0.9 (Ours) & 0.6 $\pm$ 0.4 & 0.1  & 5225.7 $\pm$ 707.6 & 13769.7 $\pm$ 9367.3 & 5050.2 $\pm$ 672.0 \\
CS-NTA-Logistic & 2.8 $\pm$ 0.9 & 1.0  & 35531.6 $\pm$ 3589.3 & 13576.9 $\pm$ 3247.8 & 13576.9 $\pm$ 3247.8 \\
Mul-Q=0.95 (Ours) & 0.6 $\pm$ 0.4 & 0.1  & 4929.4 $\pm$ 836.5 & 13114.9 $\pm$ 9163.7 & 4763.2 $\pm$ 805.0 \\
MulLog-Q=0.95 (Ours) & 0.7 $\pm$ 0.3 & 0.1  & 7729.2 $\pm$ 821.8 & 12347.3 $\pm$ 5483.3 & 7506.8 $\pm$ 1068.8 \\
CS-NTA-XGBoost & 4.9 $\pm$ 0.6 & 1.0  & 49860.0 $\pm$ 198.7 & 10366.3 $\pm$ 1347.4 & 10366.3 $\pm$ 1347.4 \\
CS-NTA-SVC-RBF & 2.8 $\pm$ 0.9 & 1.0  & 25871.3 $\pm$ 5003.9 & 9720.9 $\pm$ 2215.5 & 9720.9 $\pm$ 2215.5 \\
MulLog-Q=0.1 (Ours) & 1.8 $\pm$ 0.4 & 1.0  & 12827.8 $\pm$ 956.9 & 7511.1 $\pm$ 1225.3 & 7511.1 $\pm$ 1225.3 \\
CS-Logistic & 7.9 $\pm$ 0.3 & 1.0  & 54572.9 $\pm$ 13.0 & 6950.4 $\pm$ 238.6 & 6950.4 $\pm$ 238.6 \\
Mul-Q=0.1 (Ours) & 1.6 $\pm$ 0.4 & 1.0  & 10141.4 $\pm$ 978.2 & 6687.7 $\pm$ 1219.1 & 6687.7 $\pm$ 1219.1 \\
Naive & 6.7 $\pm$ 0.3 & 1.0  & 43402.1 $\pm$ 197.4 & 6470.2 $\pm$ 289.3 & 6470.2 $\pm$ 289.3 \\
Easy Ensemble & 8.6 $\pm$ 0.4 & 1.0  & 54992.1 $\pm$ 55.8 & 6409.1 $\pm$ 325.3 & 6409.1 $\pm$ 325.3 \\
Balanced Random Forest & 9.5 $\pm$ 0.5 & 1.0  & 55529.0 $\pm$ 54.6 & 5873.5 $\pm$ 332.2 & 5873.5 $\pm$ 332.2 \\
Balanced Bagging Decision & 11.0 $\pm$ 0.6 & 1.0  & 55731.8 $\pm$ 56.3 & 5060.6 $\pm$ 256.7 & 5060.6 $\pm$ 256.7 \\
CS-SVC-RBF & 13.2 $\pm$ 1.2 & 1.0  & 52652.4 $\pm$ 245.2 & 4023.1 $\pm$ 357.5 & 4023.1 $\pm$ 357.5 \\
NP-Random Forest & 17.6 $\pm$ 2.4 & 1.0  & 56713.9 $\pm$ 161.9 & 3284.0 $\pm$ 474.9 & 3284.0 $\pm$ 474.9 \\
CS-XGBoost & 17.8 $\pm$ 0.4 & 1.0  & 56855.2 $\pm$ 0.5 & 3199.0 $\pm$ 68.8 & 3199.0 $\pm$ 68.8 \\
RC-XGBoost & 21.8 $\pm$ 1.0 & 1.0  & 56858.0 $\pm$ 2.4 & 2615.6 $\pm$ 122.3 & 2615.6 $\pm$ 122.3 \\
RC-Logistic & 23.2 $\pm$ 3.3 & 1.0  & 56829.8 $\pm$ 15.4 & 2490.5 $\pm$ 320.1 & 2490.5 $\pm$ 320.1 \\
Isolation Forest & 63.5 $\pm$ 4.7 & 1.0  & 47993.2 $\pm$ 3420.8 & 756.9 $\pm$ 27.1 & 756.9 $\pm$ 27.1 \\
\bottomrule
\end{tabular}

\end{center}
\end{table}

\newpage
\begin{table}[!h]
    \label{table:ScreenCervical}
\caption{ScreenCervical benchmark results. We report the mean $\pm$ one standard deviation, computed over 10 repetitions with different random seeds, where each repetition corresponds to the mean performance across a 5-fold cross-validation. Method names in bold indicate that the number of false negatives was at most 1 across the 10 repetitions. 
P(FN>1) records the frequency with which the average number of false negatives (FN) across the 5-fold cross-validation exceeds one over repeated runs.} \label{Breast_complete}
\begin{center}
\begin{tabular}{llllll}
\toprule
Dataset & \multicolumn{5}{c}{ScreenCervical} \\
 & FN & P(FN>1) & TN & TN/max(FN,0.1) & TN/max(FN,1) \\
\midrule
\textbf{Naive} & 0.2 $\pm$ 0.1 & 0.0  & 12.6 $\pm$ 1.4 & 64.8 $\pm$ 17.3 & 12.6 $\pm$ 1.4 \\
\textbf{MulLog-Q=0.95 (Ours)} & 0.3 $\pm$ 0.2 & 0.0  & 10.3 $\pm$ 2.1 & 39.2 $\pm$ 13.8 & 10.3 $\pm$ 2.1 \\
\textbf{NP-Logistic} & 0.4 $\pm$ 0.2 & 0.0  & 12.4 $\pm$ 5.3 & 36.5 $\pm$ 18.1 & 12.4 $\pm$ 5.3 \\
\textbf{Mul-Q=0.95 (Ours)} & 0.3 $\pm$ 0.2 & 0.0  & 8.5 $\pm$ 1.9 & 32.0 $\pm$ 9.4 & 8.5 $\pm$ 1.9 \\
\textbf{Mul-Q=0.9 (Ours)} & 0.3 $\pm$ 0.2 & 0.0  & 8.7 $\pm$ 2.0 & 30.4 $\pm$ 9.4 & 8.7 $\pm$ 2.0 \\
RC-Logistic & 2.5 $\pm$ 1.4 & 0.8  & 64.1 $\pm$ 33.4 & 28.9 $\pm$ 12.2 & 25.3 $\pm$ 5.5 \\
\textbf{Log-Q=0.1 (Ours)} & 0.4 $\pm$ 0.2 & 0.0  & 11.8 $\pm$ 4.5 & 28.2 $\pm$ 8.9 & 11.8 $\pm$ 4.5 \\
Easy Ensemble & 3.1 $\pm$ 0.4 & 1.0  & 85.1 $\pm$ 3.0 & 27.4 $\pm$ 3.2 & 27.4 $\pm$ 3.2 \\
MulLog-Q=0.1 (Ours) & 0.8 $\pm$ 0.2 & 0.1  & 20.1 $\pm$ 4.4 & 25.7 $\pm$ 5.2 & 19.6 $\pm$ 4.1 \\
NP-Random Forest & 4.9 $\pm$ 0.9 & 1.0  & 122.9 $\pm$ 15.6 & 25.5 $\pm$ 3.0 & 25.5 $\pm$ 3.0 \\
CS-NTA-XGBoost & 5.0 $\pm$ 1.5 & 1.0  & 117.1 $\pm$ 29.1 & 24.4 $\pm$ 3.3 & 24.4 $\pm$ 3.3 \\
CS-SVC-RBF & 3.4 $\pm$ 0.5 & 1.0  & 78.5 $\pm$ 2.9 & 23.3 $\pm$ 3.1 & 23.3 $\pm$ 3.1 \\
CS-XGBoost & 6.7 $\pm$ 0.1 & 1.0  & 155.6 $\pm$ 1.3 & 23.2 $\pm$ 0.5 & 23.2 $\pm$ 0.5 \\
RC-XGBoost & 5.8 $\pm$ 0.6 & 1.0  & 133.9 $\pm$ 4.8 & 23.1 $\pm$ 1.7 & 23.1 $\pm$ 1.7 \\
\textbf{Mul-Q=0.1 (Ours)} & 0.5 $\pm$ 0.2 & 0.0  & 11.3 $\pm$ 2.5 & 23.0 $\pm$ 4.5 & 11.3 $\pm$ 2.5 \\
Balanced Bagging Decision & 5.6 $\pm$ 0.4 & 1.0  & 127.0 $\pm$ 1.8 & 22.9 $\pm$ 1.7 & 22.9 $\pm$ 1.7 \\
Isolation Forest & 4.5 $\pm$ 0.6 & 1.0  & 100.8 $\pm$ 14.8 & 22.5 $\pm$ 2.2 & 22.5 $\pm$ 2.2 \\
Balanced Random Forest & 5.8 $\pm$ 0.4 & 1.0  & 127.1 $\pm$ 3.9 & 22.0 $\pm$ 1.2 & 22.0 $\pm$ 1.2 \\
\textbf{CS-NTA-Logistic} & 0.3 $\pm$ 0.1 & 0.0  & 4.9 $\pm$ 2.1 & 21.6 $\pm$ 11.7 & 4.9 $\pm$ 2.1 \\
CS-NTA-SVC-RBF & 4.5 $\pm$ 0.9 & 1.0  & 94.5 $\pm$ 18.3 & 21.0 $\pm$ 2.2 & 21.0 $\pm$ 2.2 \\
\textbf{Conservative Ensemble} & 0.1 $\pm$ 0.1 & 0.0  & 2.5 $\pm$ 1.7 & 17.8 $\pm$ 11.4 & 2.5 $\pm$ 1.7 \\
\textbf{Log-Q=0.95 (Ours)} & 0.1 $\pm$ 0.1 & 0.0  & 2.7 $\pm$ 1.3 & 15.8 $\pm$ 5.7 & 2.7 $\pm$ 1.3 \\
\textbf{CS-Logistic} & 0.2 $\pm$ 0.0 & 0.0  & 2.7 $\pm$ 0.6 & 13.5 $\pm$ 2.8 & 2.7 $\pm$ 0.6 \\
\bottomrule
\end{tabular}

\end{center}
\end{table}

\newpage
\begin{table}[!h]
    \label{table:BreastCancer}
\caption{Breast Cancer dataset results. We report the mean $\pm$ one standard deviation, computed over 10 repetitions with different random seeds, where each repetition corresponds to the mean performance across a 5-fold cross-validation. Method names in bold indicate that the number of false negatives was at most 1 across the 10 repetitions.
P(FN>1) records the frequency with which the average number of false negatives (FN) across the 5-fold cross-validation exceeds one over repeated runs.} \label{Breast_complete}
\begin{center}
\begin{tabular}{llllll}
\toprule
Dataset & \multicolumn{5}{c}{Breast Cancer} \\
 & FN & P(FN>1) & TN & TN/max(FN,0.1) & TN/max(FN,1) \\
\midrule
\textbf{Conservative Ensemble} & 0.3 $\pm$ 0.1 & 0.0  & 53.8 $\pm$ 1.7 & 211.2 $\pm$ 77.9 & 53.8 $\pm$ 1.7 \\
\textbf{Log-Q=0.1 (Ours)} & 0.3 $\pm$ 0.1 & 0.0  & 51.4 $\pm$ 0.8 & 176.1 $\pm$ 73.0 & 51.4 $\pm$ 0.8 \\
\textbf{NP-Logistic} & 0.4 $\pm$ 0.2 & 0.0  & 47.7 $\pm$ 3.7 & 149.9 $\pm$ 69.6 & 47.7 $\pm$ 3.7 \\
\textbf{MulLog-Q=0.1 (Ours)} & 0.5 $\pm$ 0.2 & 0.0  & 56.1 $\pm$ 0.9 & 121.8 $\pm$ 31.3 & 56.1 $\pm$ 0.9 \\
\textbf{CS-NTA-Logistic} & 0.5 $\pm$ 0.2 & 0.0  & 56.0 $\pm$ 1.9 & 121.2 $\pm$ 62.6 & 56.0 $\pm$ 1.9 \\
\textbf{Log-Q=0.95 (Ours)} & 0.2 $\pm$ 0.1 & 0.0  & 19.8 $\pm$ 1.6 & 120.4 $\pm$ 48.9 & 19.8 $\pm$ 1.6 \\
\textbf{CS-Logistic} & 0.6 $\pm$ 0.2 & 0.0  & 56.4 $\pm$ 1.1 & 117.2 $\pm$ 64.6 & 56.4 $\pm$ 1.1 \\
\textbf{Mul-Q=0.9 (Ours)} & 0.3 $\pm$ 0.3 & 0.0  & 19.4 $\pm$ 2.2 & 91.8 $\pm$ 56.3 & 19.4 $\pm$ 2.2 \\
\textbf{MulLog-Q=0.95 (Ours)} & 0.4 $\pm$ 0.3 & 0.0  & 26.8 $\pm$ 1.6 & 84.6 $\pm$ 42.7 & 26.8 $\pm$ 1.6 \\
\textbf{Mul-Q=0.95 (Ours)} & 0.3 $\pm$ 0.2 & 0.0  & 16.2 $\pm$ 2.8 & 75.2 $\pm$ 43.5 & 16.2 $\pm$ 2.8 \\
NP-Random Forest & 1.1 $\pm$ 0.6 & 0.5  & 61.3 $\pm$ 4.3 & 65.5 $\pm$ 32.1 & 51.7 $\pm$ 10.2 \\
CS-NTA-XGBoost & 1.3 $\pm$ 0.4 & 0.7  & 64.6 $\pm$ 1.2 & 60.6 $\pm$ 37.5 & 50.8 $\pm$ 11.8 \\
CS-SVC-RBF & 1.1 $\pm$ 0.3 & 0.5  & 61.5 $\pm$ 1.4 & 57.4 $\pm$ 13.2 & 54.3 $\pm$ 8.4 \\
Mul-Q=0.1 (Ours) & 1.1 $\pm$ 0.4 & 0.6  & 48.5 $\pm$ 0.9 & 47.8 $\pm$ 17.3 & 40.9 $\pm$ 8.2 \\
CS-XGBoost & 1.6 $\pm$ 0.3 & 1.0  & 67.0 $\pm$ 0.4 & 42.3 $\pm$ 8.6 & 42.3 $\pm$ 8.6 \\
Easy Ensemble & 1.9 $\pm$ 0.3 & 1.0  & 69.4 $\pm$ 0.6 & 36.9 $\pm$ 5.7 & 36.9 $\pm$ 5.7 \\
CS-NTA-SVC-RBF & 2.2 $\pm$ 0.9 & 0.8  & 63.2 $\pm$ 3.4 & 34.1 $\pm$ 15.5 & 34.1 $\pm$ 15.5 \\
Balanced Random Forest & 2.1 $\pm$ 0.3 & 1.0  & 68.7 $\pm$ 0.2 & 33.5 $\pm$ 3.8 & 33.5 $\pm$ 3.8 \\
RC-Logistic & 3.2 $\pm$ 1.7 & 0.8  & 68.3 $\pm$ 2.3 & 31.9 $\pm$ 24.7 & 30.3 $\pm$ 21.1 \\
Naive & 1.7 $\pm$ 0.4 & 1.0  & 52.6 $\pm$ 0.4 & 31.7 $\pm$ 6.4 & 31.7 $\pm$ 6.4 \\
Balanced Bagging Decision & 2.7 $\pm$ 0.6 & 1.0  & 68.7 $\pm$ 0.5 & 26.7 $\pm$ 5.6 & 26.7 $\pm$ 5.6 \\
RC-XGBoost & 2.9 $\pm$ 0.8 & 1.0  & 68.4 $\pm$ 2.0 & 25.0 $\pm$ 7.0 & 25.0 $\pm$ 7.0 \\
Isolation Forest & 4.7 $\pm$ 1.0 & 1.0  & 36.0 $\pm$ 4.6 & 7.8 $\pm$ 1.1 & 7.8 $\pm$ 1.1 \\
\bottomrule
\end{tabular}

\end{center}
\end{table}
\newpage

\begin{table}[!h]
    \label{table:Golub}
\caption{Golub dataset results. We report the mean $\pm$ one standard deviation, computed over 10 repetitions with different random seeds, where each repetition corresponds to the mean performance across a 5-fold cross-validation. Method names in bold indicate that the number of false negatives was at most 1 across the 10 repetitions.
P(FN>1) records the frequency with which the average number of false negatives (FN) across the 5-fold cross-validation exceeds one over repeated runs.} \label{Breast_complete}
\begin{center}
\begin{tabular}{llllll}
\toprule
Dataset & \multicolumn{5}{c}{Golub} \\
 & FN & P(FN>1) & TN & TN/max(FN,0.1) & TN/max(FN,1) \\
\midrule
\textbf{CS-XGBoost} & 0.0 $\pm$ 0.1 & 0.0  & 8.7 $\pm$ 0.2 & 80.8 $\pm$ 20.8 & 8.7 $\pm$ 0.2 \\
\textbf{CS-NTA-XGBoost} & 0.0 $\pm$ 0.1 & 0.0  & 7.6 $\pm$ 0.5 & 72.4 $\pm$ 13.2 & 7.6 $\pm$ 0.5 \\
\textbf{Log-Q=0.1 (Ours)} & 0.0 $\pm$ 0.1 & 0.0  & 7.8 $\pm$ 0.2 & 70.0 $\pm$ 15.9 & 7.8 $\pm$ 0.2 \\
\textbf{Conservative Ensemble} & 0.0 $\pm$ 0.0 & 0.0  & 6.1 $\pm$ 0.5 & 61.2 $\pm$ 4.8 & 6.1 $\pm$ 0.5 \\
\textbf{Log-Q=0.95 (Ours)} & 0.2 $\pm$ 0.1 & 0.0  & 8.6 $\pm$ 0.2 & 47.5 $\pm$ 13.6 & 8.6 $\pm$ 0.2 \\
\textbf{Easy Ensemble} & 0.2 $\pm$ 0.2 & 0.0  & 9.1 $\pm$ 0.1 & 45.5 $\pm$ 26.4 & 9.1 $\pm$ 0.1 \\
\textbf{CS-Logistic} & 0.2 $\pm$ 0.0 & 0.0  & 8.2 $\pm$ 0.3 & 41.0 $\pm$ 1.7 & 8.2 $\pm$ 0.3 \\
\textbf{Balanced Random Forest} & 0.3 $\pm$ 0.1 & 0.0  & 9.2 $\pm$ 0.1 & 38.5 $\pm$ 12.7 & 9.2 $\pm$ 0.1 \\
\textbf{CS-NTA-Logistic} & 0.2 $\pm$ 0.1 & 0.0  & 7.2 $\pm$ 0.4 & 34.2 $\pm$ 6.3 & 7.2 $\pm$ 0.4 \\
\textbf{RC-XGBoost} & 0.4 $\pm$ 0.3 & 0.0  & 9.1 $\pm$ 0.3 & 29.7 $\pm$ 17.0 & 9.1 $\pm$ 0.3 \\
\textbf{Mul-Q=0.95 (Ours)} & 0.0 $\pm$ 0.0 & 0.0  & 2.6 $\pm$ 0.4 & 25.6 $\pm$ 4.4 & 2.6 $\pm$ 0.4 \\
\textbf{MulLog-Q=0.95 (Ours)} & 0.0 $\pm$ 0.0 & 0.0  & 2.5 $\pm$ 0.4 & 25.4 $\pm$ 4.4 & 2.5 $\pm$ 0.4 \\
Balanced Bagging Decision & 0.6 $\pm$ 0.4 & 0.1  & 9.1 $\pm$ 0.2 & 19.9 $\pm$ 10.9 & 8.7 $\pm$ 1.1 \\
\textbf{NP-Logistic} & 0.5 $\pm$ 0.3 & 0.0  & 7.2 $\pm$ 0.8 & 18.7 $\pm$ 9.7 & 7.2 $\pm$ 0.8 \\
NP-Random Forest & 0.6 $\pm$ 0.4 & 0.1  & 7.2 $\pm$ 0.9 & 15.9 $\pm$ 8.5 & 7.0 $\pm$ 1.0 \\
\textbf{Mul-Q=0.9 (Ours)} & 0.0 $\pm$ 0.0 & 0.0  & 1.4 $\pm$ 0.4 & 14.4 $\pm$ 4.1 & 1.4 $\pm$ 0.4 \\
RC-Logistic & 1.8 $\pm$ 0.8 & 0.8  & 9.3 $\pm$ 0.4 & 8.7 $\pm$ 11.1 & 5.5 $\pm$ 2.0 \\
CS-SVC-RBF & 1.4 $\pm$ 0.3 & 0.9  & 7.3 $\pm$ 0.6 & 5.8 $\pm$ 1.9 & 5.4 $\pm$ 0.9 \\
CS-NTA-SVC-RBF & 1.5 $\pm$ 0.7 & 0.8  & 7.1 $\pm$ 0.8 & 5.6 $\pm$ 2.3 & 5.0 $\pm$ 1.2 \\
Isolation Forest & 2.4 $\pm$ 0.8 & 0.9  & 5.6 $\pm$ 1.6 & 2.3 $\pm$ 0.4 & 2.3 $\pm$ 0.4 \\
\textbf{Mul-Q=0.1 (Ours)} & 0.0 $\pm$ 0.0 & 0.0  & 0.0 $\pm$ 0.0 & 0.0 $\pm$ 0.0 & 0.0 $\pm$ 0.0 \\
\textbf{MulLog-Q=0.1 (Ours)} & 0.0 $\pm$ 0.0 & 0.0  & 0.0 $\pm$ 0.0 & 0.0 $\pm$ 0.0 & 0.0 $\pm$ 0.0 \\
\textbf{Naive} & 0.0 $\pm$ 0.0 & 0.0  & 0.0 $\pm$ 0.0 & 0.0 $\pm$ 0.0 & 0.0 $\pm$ 0.0 \\
\bottomrule
\end{tabular}

\end{center}
\end{table}

\newpage


\newpage

\subsection{Feature selection results}
\label{appendix:feature_selection}
The features are 

\begin{itemize}
     \item \textbf{C\_CEIS}: Indicator related to the public employee's registration in the CEIS database (Brazilian registry of sanctioned or disqualified entities). 
     \item \textbf{C\_TCDFrestriction}: Indicator of whether the employee is subject to restrictions imposed by the Federal District Audit Court (TCDF). 
     \item \textbf{C\_CEPIM}: Indicator related to the CEPIM database, which records entities prevented from receiving public transfers. 
     \item \textbf{E\_Salary}: Gross salary in Brazilian reais (BRL), including any accumulation of multiple public positions permitted by the Constitution.
      \item \textbf{E\_SalaryMinusTax}: Net salary after deduction of taxes and mandatory social contributions. 
      \item \textbf{E\_QtySIGRHOff}: Total number of positions held by the employee in the SIGRH database up to November 2020. 
      \item \textbf{E\_QtySIAPEOff}: Total number of positions held in the public security sector according to the SIAPE database up to November 2020. 
      \item \textbf{E\_QtySIGSIPOff}: Total number of positions held across the combined SIGRH and SIAPE databases up to November 2020. 
      \item \textbf{E\_QtySIGRHfunc}: Number of distinct functions or roles occupied by the employee in the SIGRH database up to November 2020. 
      \item \textbf{E\_QtySIAPEfunc}: Number of distinct functions or roles occupied in the public security sector according to the SIAPE database up to November 2020.
       \item \textbf{E\_QtySIGSIPfunc}: Total number of functions or roles occupied across the combined SIGRH and SIAPE databases up to November 2020. 
       \item \textbf{B\_OwnershipPerc}: Percentage of company ownership held by the employee as of November 2020.
        \item \textbf{B\_QtFirmAct}: Number of secondary business activities declared by the company. 
        \item \textbf{B\_DaysOwnership}: Number of days the employee has been a shareholder or partner in the company up to November 2020. 
    \end{itemize}


\begin{table}[ht] 
    \centering 
    \caption{Feature importance/coefficient summary across methods on the Corruption dataset. EBC-Min records the frequency with which the minimum of the negative sample distribution is significantly lower than that of the positive sample distribution across train-test splits. EBC-Max records the analogous frequency after applying $x\mapsto -x$ to the data. 
    The EBC Ratio is defined as the TN/max(FN,0.1) when using only one feature, this quantifies the importance of the feature for extreme thresholding.} 
    \begin{tabular}{lccccccc} 
        \hline Feature & EBC-Max & EBC-Min & EBC Ratio & XGBoost & Logistic & Decision Tree & Random Forest \\
         \hline C\_CEIS & 0.00 & 0.00 & 0.00 & 0.0350 & 0.0822 & 0.0000 & 0.0007 \\
          C\_TCDFrestriction & 0.00 & 0.00 & 0.00 & 0.0202 & -0.3625 & 0.0000 & 0.0003 \\
           C\_CEPIM & 0.00 & 0.00 & 0.00 & 0.0000 & 0.0000 & 0.0000 & 0.0000 \\
            E\_Salary & 0.00 & 0.00 & 0.00 & 0.1131 & 0.2318 & 0.2757 & 0.3340 \\
             E\_SalaryMinusTax & 0.00 & 0.00 & 0.00 & 0.1156 & -0.3154 & 0.3573 & 0.3461 \\ 
             E\_QtySIGRHOff & 1.00 & 0.00 & 971.25 & 0.0736 & -0.2487 & 0.0451 & 0.0462 \\ 
             E\_QtySIAPEOff & 1.00 & 0.00 & 1060.00 & 0.0949 & 0.1441 & 0.0314 & 0.0259 \\
              E\_QtySIGSIPOff & 1.00 & 0.00 & 1726.50 & 0.0764 & -0.1389 & 0.1002 & 0.0547 \\
               E\_QtySIGRHfunc & 0.00 & 1.00 & 2204.00 & 0.0789 & 0.2043 & 0.0440 & 0.0280 \\
                E\_QtySIAPEfunc & 1.00 & 0.00 & 312.50 & 0.0831 & -0.0173 & 0.0017 & 0.0044 \\ 
                E\_QtySIGSIPfunc & 0.00 & 1.00 & 2124.00 & 0.0674 & 0.1916 & 0.0387 & 0.0277 \\ 
                B\_OwnershipPerc & 1.00 & 0.00 & 0.00 & 0.0767 & 0.2879 & 0.0070 & 0.0360 \\
                 B\_QtFirmAct & 0.00 & 1.00 & 0.00 & 0.0920 & 0.1439 & 0.0576 & 0.0336 \\
                  B\_DaysOwnership & 0.05 & 1.00 & 94.25 & 0.0732 & 0.0877 & 0.0412 & 0.0624 \\
                   \hline \end{tabular} \label{tab:corruption_feature_importance} 
                \end{table}

                We group them by category to compare the different feature selection methods:
\begin{itemize} 
    \item \textbf{(Number)} \textbf{E\_QtySIGRHOff}, \textbf{E\_QtySIAPEOff}, and \textbf{E\_QtySIGSIPOff}: 
    Measures of the total number of public positions held by an employee throughout their career. 
    \item \textbf{(Diversity)} \textbf{E\_QtySIGRHfunc}, \textbf{E\_QtySIAPEfunc}, and \textbf{E\_QtySIGSIPfunc}: Measures of the diversity of functions,
     responsibilities, or roles occupied by an employee. 
     \item \textbf{(Fraction)} \textbf{B\_OwnershipPerc}: Fraction of a company's equity owned by the employee.
      \item \textbf{(Secondary)} \textbf{B\_QtFirmAct}: Number of secondary business activities declared by the company, reflecting the diversity of its operations. 
      \item \textbf{(Duration)} \textbf{B\_DaysOwnership}: Duration of the employee's ownership relationship with the company, measured as the number of days since becoming a shareholder or partner.
       \item \textbf{(Salary)} \textbf{E\_Salary} and \textbf{E\_SalaryMinusTax}: Gross and net compensation received by the employee, respectively.
        \item \textbf{(Restriction)} \textbf{C\_TCDFrestriction}: Indicator of whether the employee is subject to restrictions imposed by the Federal District Audit Court (TCDF).
     \end{itemize}

             \begin{table}[ht] 
                \centering 
                \begin{tabular}{ccccccc} \hline Rank & EBC Ratio & XGBoost & Logistic & Decision Tree & Random Forest & Highest Frequency Summary \\
                     \hline 1 & Diversity & Salary & Restriction & Salary & Salary & Salary \\
                      2 & Number & Salary & Salary & Salary & Salary & Salary \\
                      3 & Duration & Number & Fraction & Number & Duration & Number \\
                     4 & Secondary & Secondary & Salary & Secondary & Number & Secondary \\
                      5 & Fraction & Diversity & Number & Number & Number & Fraction \\ 
                      \hline \end{tabular}
                    \caption{Top 5 categories ranked by absolute importance or coefficient magnitude for each method. 
                    Salary refers to the remuneration-related variables (Salary and SalaryMinusTax), Number to the number of public positions held, Diversity to the number of distinct functions occupied,
                     Fraction to the ownership percentage in a company, Secondary to the number of secondary business activities, Duration to the length of company ownership, 
                     and Restriction to administrative restrictions reported by oversight agencies.} 
                      \label{table:top5_corruption} 
                    \end{table}

                    \paragraph{Performances. }

                    To further assess the practical usefulness of the selected features, we evaluated the predictive performance obtained when restricting the analysis to the Top 5 five features identified by our permutation-based feature selection procedure and EBC-Ratio in \Cref{tab:corruption_feature_importance}. 
                    Using only these five variables, Extreme Thresholding (EVT) achieved $\mathrm{FN}=0.00 \pm 0.00$ and $\mathrm{TN}=5226.98 \pm 361.60$, compared with $\mathrm{FN}=0.00 \pm 0.00$ and $\mathrm{TN}=5593.93 \pm 563.11$ when all features were used.
                     This indicates that the performance loss induced by feature selection is relatively small while substantially improving interpretability.
                      For comparison, when trained on the full feature set,
                       Logistic Regression achieved $\mathrm{FN}=0.53 \pm 2.32$ and 
                       $\mathrm{TN}=1528.71 \pm 323.26$, XGBoost achieved $\mathrm{FN}=121.25 \pm 30.81$ and $\mathrm{TN}=183286.35 \pm 4885.81$, LASSO achieved $\mathrm{FN}=38.82 \pm 49.64$ 
                       and $\mathrm{TN}=33753.97 \pm 36777.37$, Decision Tree achieved $\mathrm{FN}=52.92 \pm 23.60$ and $\mathrm{TN}=74432.74 \pm 12919.35$, 
                       and Random Forest achieved $\mathrm{FN}=392.75 \pm 18.17$ and $\mathrm{TN}=320099.24 \pm 446.96$.
                        These results highlight the trade-off between reducing false negatives and increasing true negatives. 
                        While several machine-learning models achieve substantially larger TN values, 
                        they do so at the cost of a markedly higher number of false negatives.
                         In contrast, EVT maintains perfect control of FN while retaining a non-negligible TN,
                          even after restricting the model to only five selected features.

\newpage
\subsection{Computation time}
\label{subsection:computation_time}
\begin{figure*}[!h] 
    \centering 
    \includegraphics[width=0.8\textwidth]{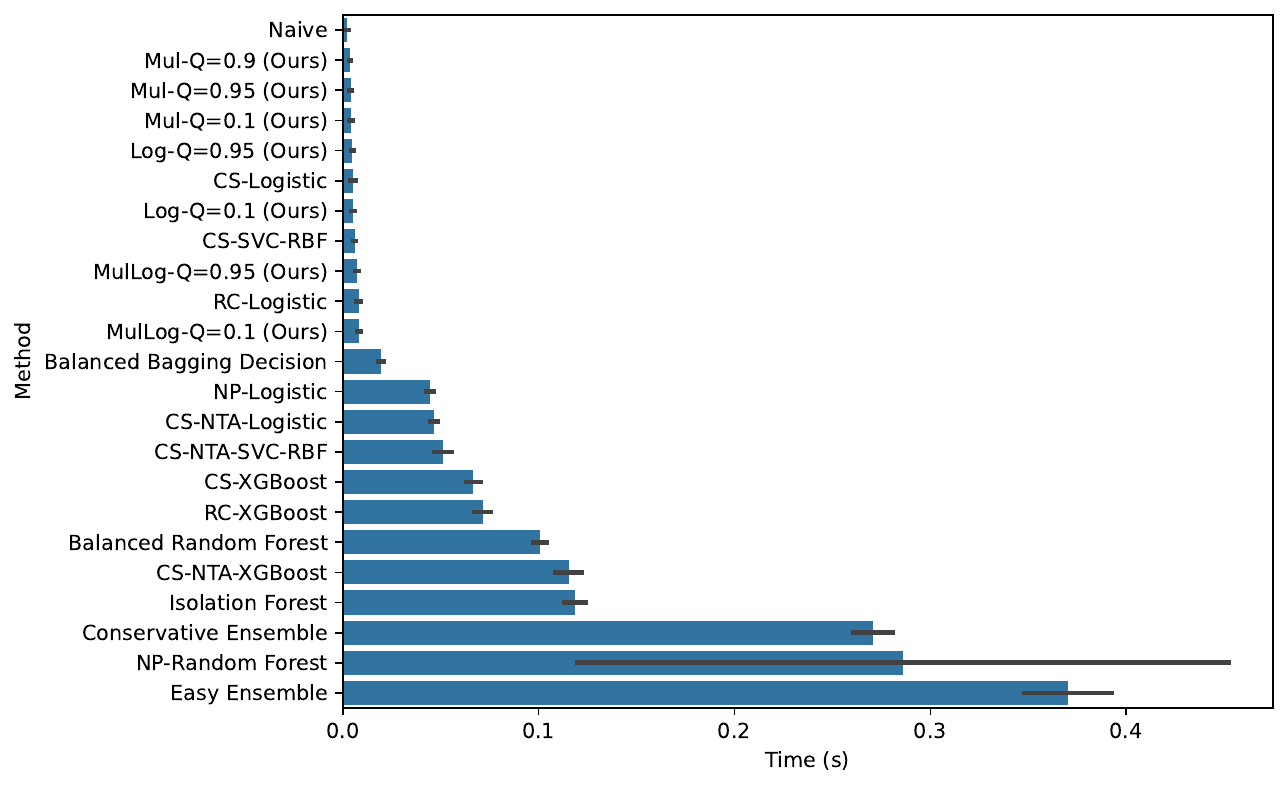}
    \caption{ Barplot of the training time for the different methods on the dataset ScreenX. Runtime measurements were obtained on a CPU-only server with 4\,GB of RAM. }
      \label{fig:time_screenx} 
    \end{figure*}
    \begin{figure*}[!h] 
    \centering 
    \includegraphics[width=0.8\textwidth]{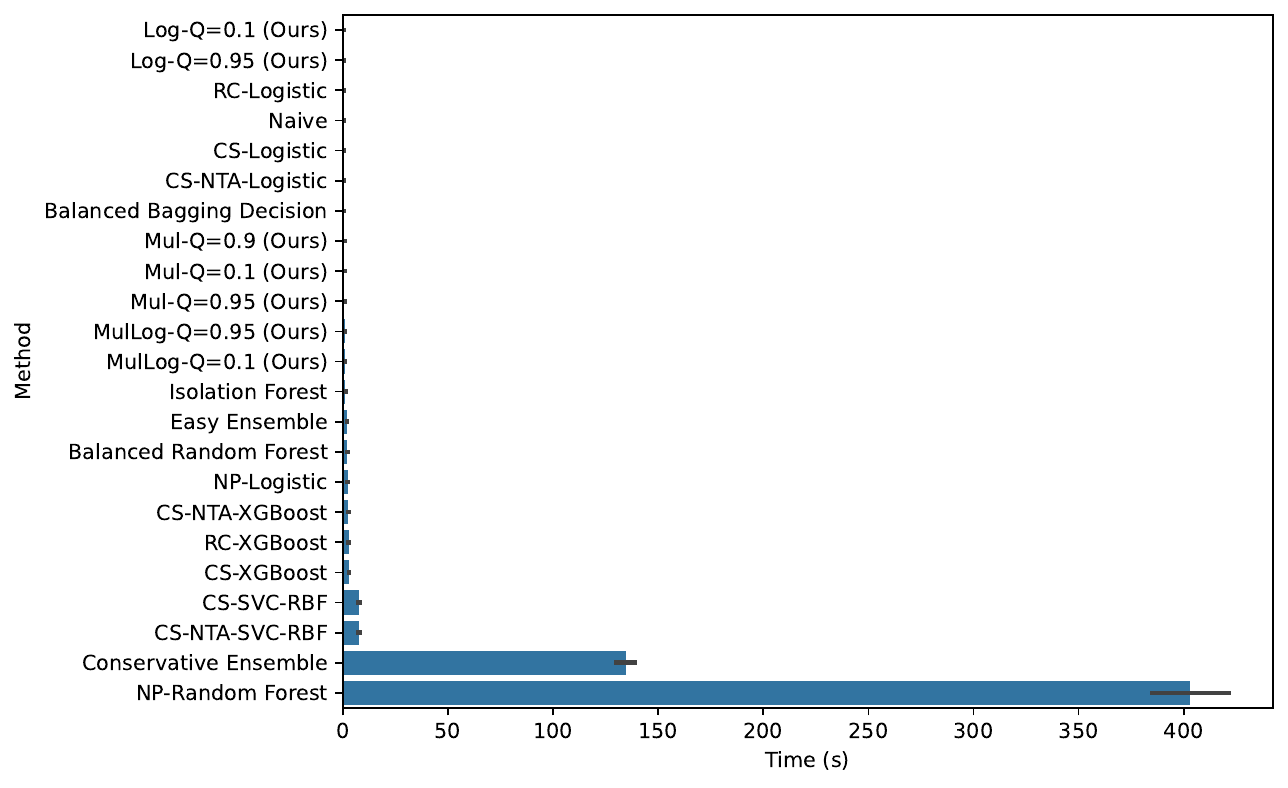}
    \caption{ Barplot of the training time for the different methods on the dataset Credit. Runtime measurements were obtained on a CPU-only server with 4\,GB of RAM. }
      \label{fig:time_credit} 
    \end{figure*}
     \begin{figure*}[!h] 
    \centering 
    \includegraphics[width=0.8\textwidth]{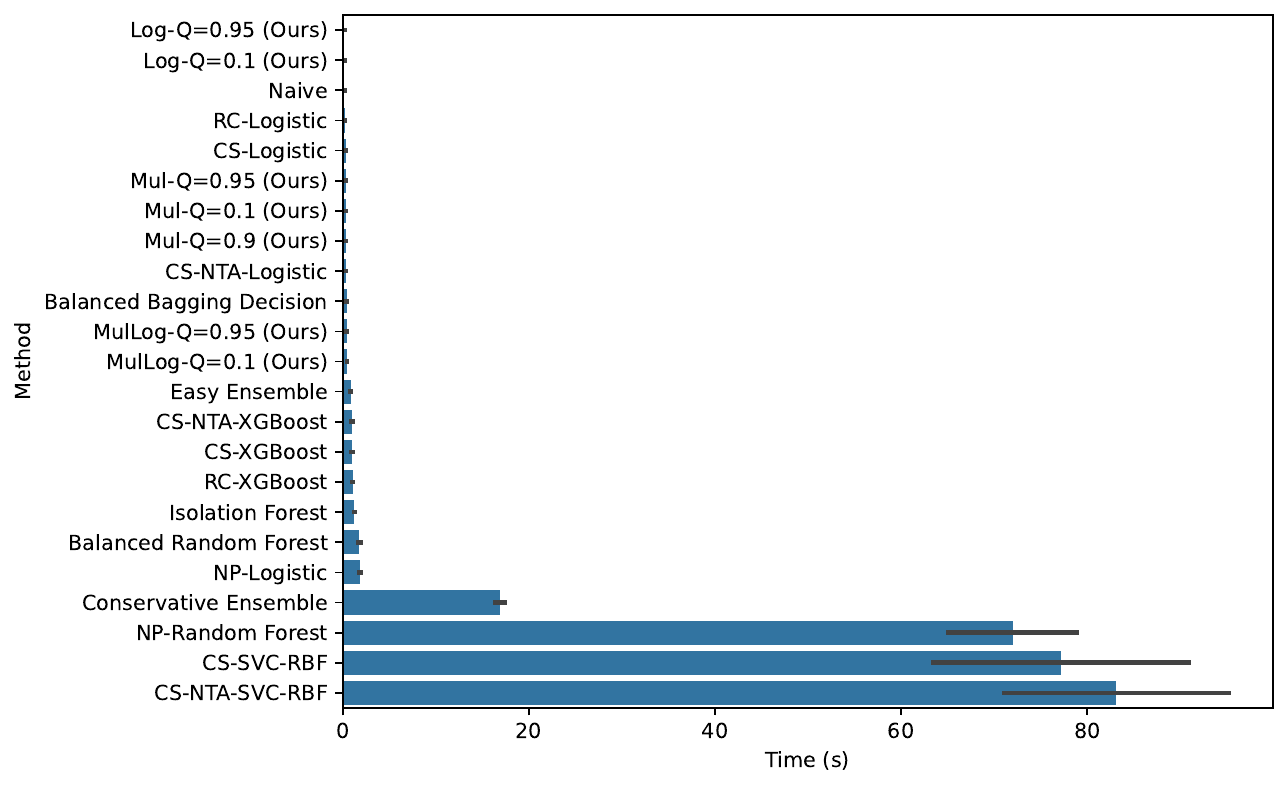}
    \caption{ Barplot of the training time for the different methods on the dataset Corruption. Runtime measurements were obtained on a CPU-only server with 4\,GB of RAM. }
      \label{fig:time_corruption} 
    \end{figure*}
    \begin{figure*}[!h] 
    \centering 
    \includegraphics[width=0.8\textwidth]{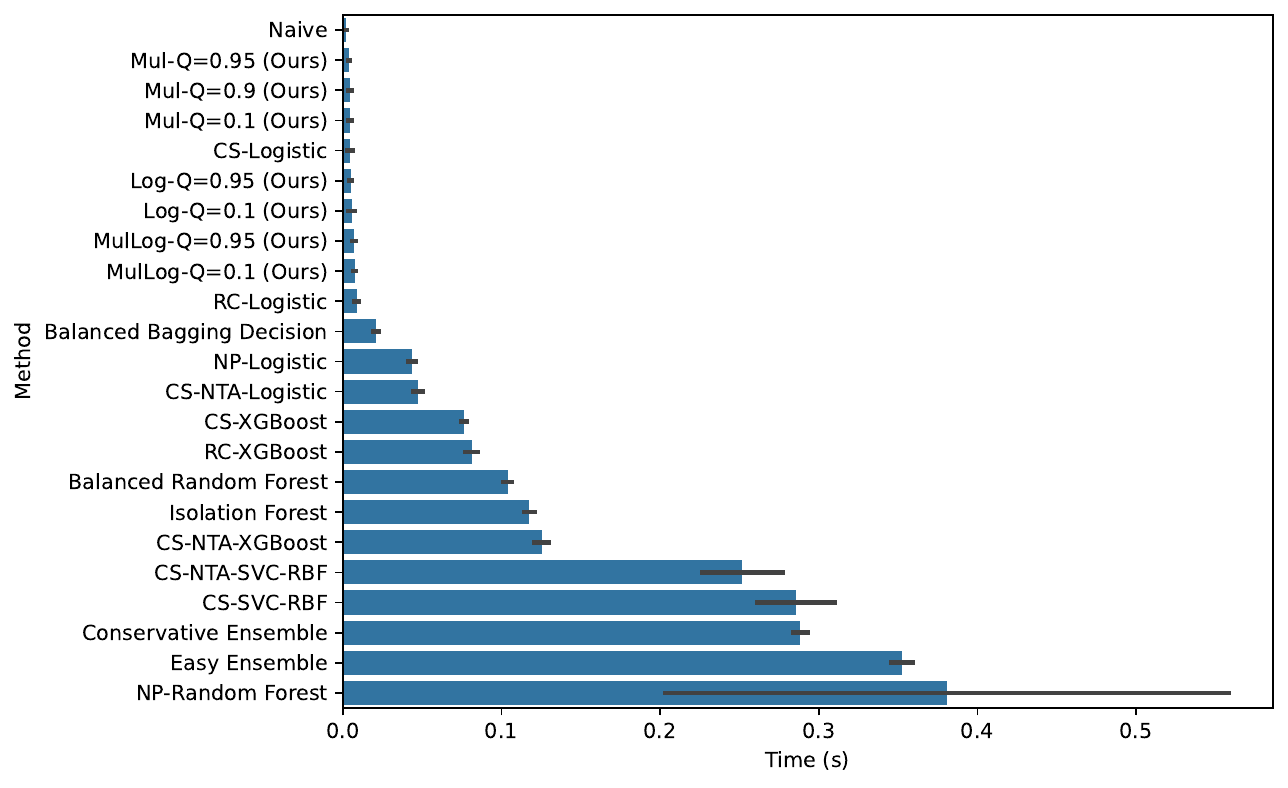}
    \caption{ Barplot of the training time for the different methods on the dataset ScreenCervical. Runtime measurements were obtained on a CPU-only server with 4\,GB of RAM. }
      \label{fig:time_screencervical} 
    \end{figure*}
    \begin{figure*}[!h] 
    \centering 
    \includegraphics[width=0.8\textwidth]{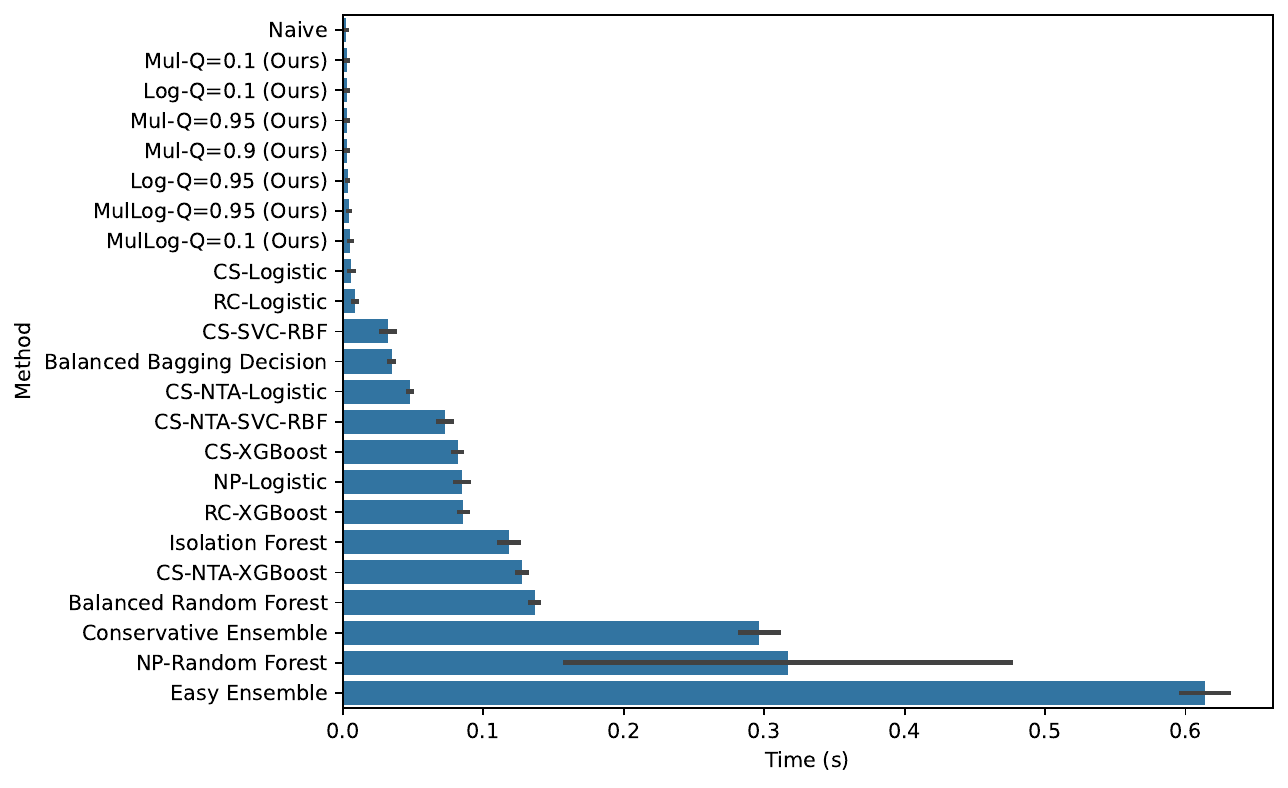}
    \caption{ Barplot of the training time for the different methods on the dataset Breast Cancer. Runtime measurements were obtained on a CPU-only server with 4\,GB of RAM. }
      \label{fig:time_breast} 
    \end{figure*}

    \begin{figure*}[!h] 
    \centering 
    \includegraphics[width=0.8\textwidth]{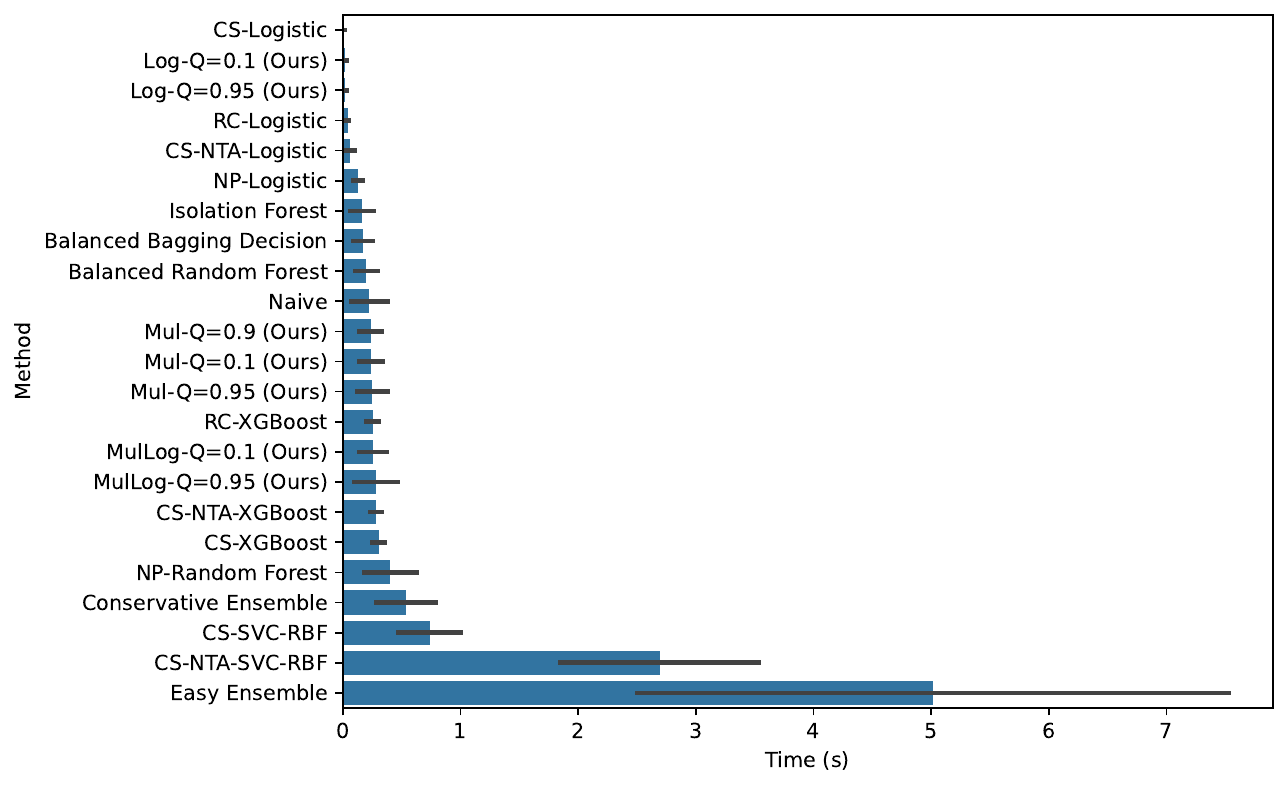}
    \caption{ Barplot of the training time for the different methods on the dataset Golub. Runtime measurements were obtained on a CPU-only server with 4\,GB of RAM. }
      \label{fig:time_golub} 
    \end{figure*}
\newpage
\section{Discussion of the Umbrella NP algorithm}
\label{discussion_NP}
In  \cite{tong2018neyman}[Proposition 1], the idea is not to ensure the false negative constraint during the training, but to validate it with high probability using a post-processing step, by proceeding as follows:
 \begin{enumerate}
    \item Split the given dataset $\mathcal{D}=\mathcal{D}_{\text{val}}\cup\mathcal{D}_{\text{train}} $ in a validation set $\mathcal{D}_{\text{val}}$ and a train set $\mathcal{D}_{\text{train}}$ such that $\mathcal{D}_{\text{val}}$ contains only features $x_i$ related to class $y_i=1$.
    \item Train a scoring function $\phi$ on the train set such that $g(X)=1-\mathbbm{1}(\phi(X)>0)$ is the related classifier which predicts $0$ when $\phi(X)>0$ and assume that $\phi(X)$ is continuous. $\phi$ can be a logistic regression score, or a support vector decision function for instance.
    \item Then, consider $T_i=\phi(x_i^\text{val})$ for any $x_i^\text{val} \in \mathcal{D}_{\text{val}}$ and rank it such that $T_{(i)}$ is the $i$-th smallest scores among $(T_i)_{i \in [N^{\text{val}}_1]}$.
    \item Given an upper bound $\alpha>0$ and a probability of violation $(1-\alpha)^{N^{\text{val}}_1}\leq \delta$, define 
    \begin{equation} 
        k^* =\min \left\{i \in [N^{\text{val}}_1]: \sum_{j=i}^{N^{\text{val}}_1} \binom{N^{\text{val}}_1}{j} \alpha^{N^{\text{val}}_1-j} (1-\alpha)^{j}\leq \delta \right\} 
    \end{equation}
    then the random classifier $g(X)=1-\mathbbm{1}(\phi(X)\geq T_{(k^*)} ) $ depending on the sample $T_{(k^*)}$ satisfies $\mathbb{P}(g(X)\neq Y|Y=1)\leq \alpha $ with probability $(1-\delta)$.
 \end{enumerate}
 Given the rank statistics $(T_{(i)})_{i \in [N^{\text{val}}_1]}$, this post-processing step finds the smallest threshold that ensures that the false negative constraint is fulfilled with probability $(1-\delta)$ by only assuming that $\phi(X)$ is a continuous random variable and i.i.d data generation.
In this approach, the false positive rate $\mathbb{P}(g(X)\neq Y|Y=0)$ is not minimized explicitly and heavily depends on the scoring function $\phi$.
Coming back to our case where $\alpha\approx 0$, we have $(1-\alpha)^{N^{\text{val}}_1}\approx 1-N^{\text{val}}_1 \alpha\leq \delta $ and thus $(1-\delta)\approx N^{\text{val}}_1 \alpha \leq N _1\alpha \ll 1 $, meaning that with this method the best classifier is $g(X)=1-\mathbbm{1}(\phi(X)\geq T_{(N^{\text{val}}_1)} ) $ and it ensures the false negative constraint with a very low probability.
Note that in dimension $d=1$, using $\phi(X)=X$, there is no need for a training set and thus $\mathcal{D}_{\text{val}}=\mathcal{D}$, the postprocessing method thus yields $g=g_{-\infty,b_{\text{naive}}}$.
\newpage
\section{Permutation test on the maximum}
\label{appendix:bootstrap}

Let two independent samples $(W_i)\in [n], (V_i)_{i\in [m]} $ such that $W_i\overset{i.i.d}{\sim} P_1 $, $V_i\overset{i.i.d}{\sim} P_2 $ then
\begin{equation}
    M_W=\max_{i\in [n]} W_i,\quad M_V=\max_{i\in [m]} V_i
\end{equation}
Define $T_{obs}=M_W-M_V$. Then, under interchangeability assumption $H_0$, the two samples are concatenated $(Z_i)_{i\in [n+m]}=(W_i)\cup(V_i)  $. 
Denoting by $\pi: [n+m]\mapsto [n+m]$ a random permutation, then define
$$W_\pi=(Z_{\pi(i)})_{i\in [n]}, V_\pi=(Z_{\pi(n+i)})_{i\in [m]},\quad T_\pi=\max W_\pi -\max V_\pi$$
the p value is defined as follows to test whether the maximum of $P_1$ is higher than the maximum of $P_2$ :
\begin{equation}
    p_{\text{perm}}=\frac{1+\sum_{\pi: [n+m]\mapsto [n+m]} \mathbbm{1}(T_\pi\geq T_{obs})}{1+k_{perm}}
\end{equation}
where $k_{perm}$ is the number of permutation.
 The test is exact 
\begin{equation}
    \mathbb{P}(p_{\text{perm}}\leq p|H_0)\leq p \eqsp .
\end{equation}

Numerically, most often we only sample a $\tilde{k}$-subset of permutation $(\pi_i)_{i\in [\tilde{k}]}$ and estimate,
\begin{equation}
    \hat{p}_{\text{perm}}=\frac{1+\sum_{i=1}^{\tilde{k}} \mathbbm{1}(T_{\pi_i}\geq T_{obs})}{1+\tilde{k}}
\end{equation} 
However, in our case, an explicit expression is tractable. Indeed, denoting by $ (Z_{i,n+m})$ the sorted samples in $ (Z_i)_{i\in [n+m]}$, the following Lemma yields a formula.
\begin{lemma}
    \label{perumation:maxima}
    Assuming that all values $(Z_i)$ are distinct, under $H_0 $,
    for any $ j\in [m,n+m-1]$
    \begin{align}
        &\mathbb{P}(\max W_\pi\leq Z_{j,n+m},\max V_\pi=Z_{n+m,n+m})=\frac{\binom{j}{n}}{\binom{n+m-1}{n}}\frac{m}{n+m}\\
        &\mathbb{P}(\max W_\pi=Z_{n+m,n+m},\max V_\pi\leq Z_{j,n+m})=\frac{\binom{j}{m}}{\binom{n+m-1}{m}}\frac{n}{n+m}\\
        &\mathbb{P}(\max W_\pi\neq Z_{n+m,n+m},\max V_\pi\neq Z_{n+m,n+m})=0
    \end{align}
    therefore, if $T_{obs}> 0$, and defining $n+m-l\geq m$ s.t $Z_{n+m,n+m}-Z_{n+m-l,n+m}=T_{obs} $,
    \begin{align}
        &\mathbb{P}(T_\pi\geq T_{obs})=\sum_{Z_{n+m,n+m}-Z_{j,n+m}\geq T_{obs}} \mathbb{P}(\max W_\pi=Z_{n+m,n+m},\max V_\pi= Z_{j,n+m}) \\
        &\mathbb{P}(T_\pi\geq T_{obs})= \mathbb{P}(\max W_\pi=Z_{n+m,n+m},\max V_\pi\leq  Z_{n+m-l,n+m}) \\
        &=\frac{n}{n+m}  \prod_{i=2}^{l} \frac{n-l-1+i}{n-l+m-1+i}
    \end{align}
\end{lemma}
\begin{proof}
    First, note that
    \begin{equation}
        \mathbb{P}(\max W_\pi\neq Z_{n+m,n+m},\max V_\pi\neq Z_{n+m,n+m})=0
    \end{equation}
    because $\pi$ is surjective.
    Second,
    \begin{equation}
        \mathbb{P}(\max W_\pi=Z_{n+m,n+m})=\mathbb{P}(l^*\in \pi([n]), \, \text{where} Z_{n+m,n+m}=Z_{l^*})=\frac{n}{n+m} .
    \end{equation}
    
    Indeed, we first select, among the elements of $[n]$, the preimage of $l^*$, yielding $n$ possible choices. Once this choice is fixed, the permutation $\pi$ is completely determined by assigning the remaining $n+m-1$ values, for which there are $(n+m-1)!$ possibilities. Dividing by the total number of permutations of $[n+m]$, namely $(n+m)!$, gives the desired result.

    Third, denoting by $V_\pi^1,\ldots, V_\pi^m $ the sorted values of $V_\pi$ which define bijectively a subset $\{i_v^1,\ldots,i_v^m\} $ of $[n+m]$,
    \begin{align}
        &\mathbb{P}(\max W_\pi=Z_{n+m,n+m},\max V_\pi\leq Z_{j,n+m})=\mathbb{P}(\max V_\pi\leq Z_{j,n+m}|\max W_\pi=Z_{n+m,n+m})\frac{n}{n+m} \\
        &=\mathbb{P}(\{i_v^1,\ldots,i_v^m\},\text{subset of} \, [j] |\max W_\pi=Z_{n+m,n+m})\frac{n}{n+m} \\
        &= \frac{\binom{j}{m}}{\binom{n+m-1}{m}}\frac{n}{n+m} \\
        &=\frac{j!(n-1)!}{(n+m-1)!(j-m)! }\frac{n}{n+m}
    \end{align}
    Using $j=n+m-l$, we have
    \begin{equation}
        \frac{j!(n-1)!}{(n+m-1)!(j-m)! }=\prod_{i=2}^k \frac{n-k-1+i}{n-k+m-1+i} \eqsp .
    \end{equation}

    The results are symmetric for $ \max V_\pi$, which concludes the proof.
\end{proof}

Note that when $T_{obs}>0$ the exceedance level $l$ in \Cref{perumation:maxima} equals $\text{Card}\{i: x_i>b_{naive}, y_i=0 \}$ which can be computed in $O(N)$ with a linear scan without sorting the negative class.

For discrete data, observations are not necessarily distinct, and therefore the formula is no longer exact.
 In such cases, we still use the same expression as an approximation. More generally, our EVT-based approach is not well suited to discrete data, as the underlying assumptions of extreme value analysis are not naturally satisfied in this setting.

Note that similar formulas were already derived in \cite{rosenbaum1954tables} in 1954. Although that work was not concerned with permutation tests based on extrema, it addressed a related anomaly detection problem.


\section{Proofs}
\label{appendix:proofs}

\subsection{Proof of \Cref{lemma:naive}}
In \eqref{eq:naive_classifier}, we have $a_{\text{naive}},b_{\text{naive}} \in \{(a,b):\,\sum_{i:y_i=1}\mathbbm{1}\!\left(g_{a,b}(x_i)\neq 1\right)=0\}$, this implies that for any $i\in [N]$ such that $y_i=1$,
\begin{equation}
    a_{\text{naive}}\leq x_i, \quad b_{\text{naive}}\geq x_i
\end{equation}
thus $a_{\text{naive}}\leq \min \{x_i:y_i=1\} $ and $b_{\text{naive}}\geq \max \{x_i:y_i=1\} $. 
Then, note that $\mathbbm{1}\!\left(g_{a_{\text{naive}},b_{\text{naive}}}(x_i)\neq y_i\right)=1$ as long as $b_{\text{naive}}\geq   x_i$ and $a_{\text{naive}}\leq  x_i $ for any $i\in[N]$ such that $y_i=0$.
 Therefore, the solutions of \eqref{eq:naive_classifier} are ,
\begin{align}
\max \{x_i:y_i=0, x_i<\min \{x_i:y_i=1\}\} < a_{\text{naive}}\leq \min \{x_i:y_i=1\} \\ 
\min \{x_i:y_i=0, x_i>\max \{x_i:y_i=1\}\} > b_{\text{naive}}\geq \max \{x_i:y_i=1\} \, .
\end{align}
$a_{\text{naive}},b_{\text{naive}}=\min \{x_i:y_i=1\},\max \{x_i:y_i=1\}$ is always a valid solution. 

\subsection{Proof of \Cref{theorem:controlling_alpha}}
\begin{proof}
    Denote by
        \begin{align}
     & P_{\text{true}}=\mathbb{P}(U_{N_1+1}>b_{\text{naive}}+\Delta_b | (U_i)_{i\in [N_1]}) \\
     &=\mathbb{P}(U_{N_1+1}>b_{\text{naive}}+\Delta_b | U_{N_1+1}\geq U_{N_1-K,N_1} , (U_i)_{i\in [N_1]})\mathbb{P}(U_{N_1+1}\geq U_{N_1-K,N_1} | (U_i)_{i\in [N_1]}),  \\ 
     &P_{\text{pareto}}= (1-G_{\xi,\beta(U_{N_1-K,N_1})}) (b_{\text{naive}}+\Delta_b-U_{N_1-K,N_1}) \frac{K}{N_1},\\
     & P_{\text{pivot}}=(1-G_{\xi,\beta(U_{N_1-K,N_1})}) (b_{\text{naive}}+\Delta_b-U_{N_1-K,N_1}) \mathbb{P}(U_{N_1+1}\geq U_{N_1-K,N_1} | (U_i)_{i\in [N_1]}) \quad .  \\
        \end{align}
        We have for any $\alpha>0$,
        \begin{align}
           & \mathbb{P}(P_{\text{true}}\leq \alpha )\geq \mathbb{P}(P_{\text{pareto}}+|P_{\text{true}}-P_{\text{pareto}}|\leq \alpha )\\ 
           &\geq \mathbb{P}((P_{\text{pareto}} \leq \alpha/2)\cap (|P_{\text{true}}-P_{\text{pareto}}|\leq \alpha/2) ) \\
            & \label{eq:alpha_bound}\geq 1-\mathbb{P}((P_{\text{pareto}}> \alpha/2) \cup |P_{\text{true}}-P_{\text{pareto}}|> \alpha/2 )
        \end{align}

        \textbf{Bounding $|P_{\text{true}}-P_{\text{pareto}}|$:}
        By triangular inequality 
        \begin{equation}
            \label{eq:triangular_true_pareto}
            |P_{\text{true}}-P_{\text{pareto}}|\leq |P_{\text{true}}-P_{\text{pivot}}|+|P_{\text{pivot}}-P_{\text{pareto}}|\eqsp .
        \end{equation}

        Denoting by $$L_{\text{quant}}=\left|\frac{K}{N_1}-\mathbb{P}\left(U_{N_1+1}\geq U_{N_1-K,N_1} | (U_i)_{i\in [N_1]}\right)\right|$$ 
        we have
        \begin{equation}
            \label{eq:triangular_true_pareto_1}
            |P_{\text{pivot}}-P_{\text{pareto}}|= (1-G_{\xi,\beta(U_{N_1-K,N_1})}) (b_{\text{naive}}+\Delta_b-U_{N_1-K,N_1}) L_{\text{quant}}
        \end{equation}

        Denoting by 
        \begin{align}
            L_{\text{mis}}=|(1-G_{\xi,\beta(U_{N_1-K,N_1})})(b_{\text{naive}}+\Delta_b-U_{N_1-K,N_1})-\mathbb{P}(U_{N_1+1}>b_{\text{naive}}\\
            +\Delta_b | U_{N_1+1}\geq U_{N_1-K,N_1} , (U_i)_{i\in [N_1]})|,
        \end{align}
        we have,
        \begin{align}
           &\label{eq:triangular_true_pareto_2}|P_{\text{true}}-P_{\text{pivot}}|\leq  L_{\text{mis}} \mathbb{P}(U_{N_1+1}\geq U_{N_1-K,N_1} | (U_i)_{i\in [N_1]}) \\
           & \leq \epsilon_0(U_{N_1-K,N_1}) \mathbb{P}(U_{N_1+1}\geq U_{N_1-K,N_1} | (U_i)_{i\in [N_1]})^2\\
        \end{align}

        since by assumption \eqref{eq:fast_converging} we have $L_{\text{mis}}\leq \epsilon_0(U_{N_1-K,N_1})\mathbb{P}(U_{N_1+1}\geq U_{N_1-K,N_1} | (U_i)_{i\in [N_1]})$.

        By  \eqref{eq:triangular_true_pareto}, \eqref{eq:triangular_true_pareto_1}, \eqref{eq:triangular_true_pareto_2} and $\mathbb{P}(U_{N_1+1}\geq U_{N_1-K,N_1} | (U_i)_{i\in [N_1]})\leq L_{\text{quant}}+K/N_1$
        \begin{equation}
            |P_{\text{true}}-P_{\text{pareto}}|\leq L_{\text{quant}} (1-G_{\xi,\beta(U_{N_1-K,N_1})}) (b_{\text{naive}}+\Delta_b^*-U_{N_1-K,N_1})+\epsilon_0(U_{N_1-K,N_1})(\fracD{K}{N_1}+L_{\text{quant}})^2\eqsp .
        \end{equation}
        
        By Dvoretzky-Kiefer-Wolfowitz theorem using i.i.d assumption, using $\delta_{\text{quant}}>0$,
        \begin{align}
        \mathbb{P}(L_{\text{quant}}> \sqrt{\frac{\log(2/\delta_{\text{quant}})}{2N_1}})\leq \delta_{\text{quant}} \eqsp.
        \end{align}
        Indeed, for any $t\in \mathbb{R}$, denoting by $F(t)=\mathbb{P}\left(U_{N_1+1}\leq t | (U_i)_{i\in [N_1]}\right)$, we have $L_{quant}=| (1-F(U_{N_1-K,N_1}))-(1-\hat{F}(U_{N_1-K,N_1}))|=|F(U_{N_1-K,N_1})-\hat{F}(U_{N_1-K,N_1})|$ where the empirical distribution fonction is $\hat{F}(t)=\sum_{i=1}^{N_1} \mathrm{1}_{U_{i,N_1}\leq t }/N_1$ and $\hat{F}(U_{N_1-K,N_1})=1-K/N_1$.

        Since the cumulative distribution is non decreasing,
        \begin{equation}
            \mathbb{P}( (1-G_{\xi,\beta(U_{N_1-K,N_1})})(b_{\text{naive}}+\Delta_b-U_{N_1-K,N_1})\leq Q_K)\geq \mathbb{P}((1-G_{\xi,\beta(U_{N_1-K,N_1})})(\Delta_b)\leq Q_K)
        \end{equation}

        and by assumption we have,
        \begin{equation}
           \mathbb{P}( (1-G_{\xi,\beta(U_{N_1-K,N_1})})(\Delta_b)\leq Q_K)\geq 1-\delta_{\text{est}}\eqsp .
        \end{equation}

        Finally, using $K=\lfloor\sqrt{N_1}\rfloor+1\geq \sqrt{N_1}$, we have with probability $1-\delta_{\text{quant}}-\delta_{\text{est}}$,
        \begin{equation}
            |P_{\text{true}}-P_{\text{pareto}}|\leq \sqrt{\frac{\log(2/\delta_{\text{quant}})}{2N_1}} Q_K+\epsilon_0(U_{N_1-K,N_1})(1+\sqrt{\frac{\log(2/\delta_{\text{quant}})}{2}})^2/N_1 \eqsp .
        \end{equation}
        Using $Q_K=\epsilon_Q (K)/K$ with $\lim_{k\to \infty}\epsilon_Q (K)= 0$, we have with probability $1-\delta_{\text{quant}}-\delta_{\text{est}}$,
        \begin{equation}
            |P_{\text{true}}-P_{\text{pareto}}|\leq  \max(\epsilon_Q (K),\epsilon_0(U_{N_1-K,N_1}))(2(1+\sqrt{\frac{\log(2/\delta_{\text{quant}})}{2}})^2/N_1) \eqsp ,
        \end{equation}
        where we used that $K\geq \sqrt{N_1}$.
        \textbf{Bounding $P_{\text{pareto}}$:}

        Using that $Q_K \frac{K}{N_1}=\epsilon_Q (K)/N_1$ , we have
        \begin{equation}
            \mathbb{P}( P_{\text{pareto}}\leq \epsilon_Q (K)/N_1)\geq 1-\delta_{\text{est}}
        \end{equation}

        \textbf{Conclusion:}
        Denoting by $\epsilon_1(N_1)=\max(\epsilon_Q (K),\epsilon_0(U_{N_1-K,N_1}))$, this yields using \eqref{eq:alpha_bound} with $\alpha=4\epsilon_1(N_1)(1+\sqrt{\frac{\log(2/\delta_{\text{quant}})}{2}})^2/N_1$,
        \begin{equation}
            \mathbb{P}(P_{\text{true}}\leq \alpha )\geq 1-\delta_{\text{quant}}-\delta_{\text{est}}\eqsp .
        \end{equation}


     \end{proof}

     \subsection{Proof of \Cref{prop:excess}}

     \begin{proof}

        First, we prove the following Lemma.

        \begin{lemma}
        \label{lemma:bound_xi_neg}
        Assume that $W$ follows a Generalized Pareto Distribution $G_{\xi,\beta}$ with parameter $\xi\leq 0$ and $\beta>0$, then 
        $ \mathbb{P}(W\geq q_{1-Q})\leq Q$ using $q_{1-Q} =-\beta\log(Q) (1+\exp(-1))/(1-\xi) $ for any $0<Q\leq \exp(-1)$.
     \end{lemma}
     \begin{proof}
     \textbf{Case $\xi=0$: } $G_{\xi,\beta}$ is an Exponential distribution of parameter $ 1/\beta$, we have $\mathbb{E}(W)=\beta$ and $q_{1-Q} =-\beta \log(Q)$ and $1\leq 1+\exp(-1)$ .

     \textbf{Case $\xi<0$: } We have $\mathbb{E}(W)=\beta/(1-\xi)$ and $q_{1-Q}^\xi =\beta (\exp(-\xi \log(Q))-1)/\xi$ is exactly the quantile of level $Q$. We have to show 
     $$ -\log(Q) (1+\exp(-1))/(1-\xi)\geq (\exp(-\xi \log(Q))-1)/\xi$$ 
     Denoting $z=-\log(Q)\geq 0$ and $\gamma=-\xi\geq 0$, and multiplying by $-\xi (1-\xi) $ we study the function $h:\gamma\in \mathbb{R}_{\geq 0}\mapsto \gamma (1+\exp(-1))z -(1+\gamma)(1-\exp(-\gamma z))$ and aim to show $h(\gamma)\geq 0$. We have
     \begin{align}
        &h'(\gamma)=(1+\exp(-1))z-(1-\exp(-\gamma z))-(1+\gamma)z\exp(-\gamma z) \\
        &h'(\gamma)=(z-1)(1-\exp(-\gamma z)) +z\exp(-1)-\gamma z \exp(-\gamma z)
     \end{align}
     Noting that $|x\exp(-x)|\leq \exp(-1) $, we have $h'(\gamma)\geq 0$ as long as $z\geq 1$. Since $h(0)=0$, we have thus $h(\gamma)\geq 0 $, which yields the inequality.
     \end{proof}

        By \Cref{lemma:bound_xi_neg}, we have for any $\xi\leq 0$, $\Delta_b^*=-\mathbb{E}_{D\sim G_{\xi,\beta(U_{N_1-K,N_1} )}}(D) (1+\exp(-1))\log(1/(K\log(K)))$ verifies \eqref{eq:estimator_cond} with $\delta_{\text{est}}=0$ since $\mathbb{E}_{D\sim G_{\xi,\beta(U_{N_1-K,N_1} )}}(D)=\beta(U_{N_1-K,N_1} )/(1-\xi)$ and $1/(K\log(K))\leq \exp(-1) $ for $K\geq 3$,
        then the result holds using \eqref{eq:concentration}.
    \end{proof}

    \subsection{Multidimensional case inflation}
    \label{appendix:multidim}
    Assuming independence of coordinates
    \begin{align}
    &\mathbb{P}(g(X)\neq 1 |Y=1)=\mathbb{P}(\bigcup_{i=1}^d (g_{a_i,b_i}(X^{i})\neq 1) |Y=1)=1-\prod_{i=1}^d \mathbb{P}(g_{a_i,b_i}(X^{i})= 1 |Y=1)\\
    &\leq 1 -(1-\alpha)^d\approx d\alpha 
    \end{align}
    Without assuming anything 
     \begin{align}
    &\mathbb{P}(g(X)\neq 1 |Y=1)=\mathbb{P}(\bigcup_{i=1}^d (g_{a_i,b_i}(X^{i})\neq 1) |Y=1)
    &\leq \sum_{i=1}^d \mathbb{P}( (g_{a_i,b_i}(X^{i})\neq 1) |Y=1)\leq d_{\text{selected}} \alpha 
    \end{align}
    where $d_{\text{selected}}$ is the number of dimensions where $g_{a_i,b_i}$ is not constant.






\end{document}